\documentclass{article}

\usepackage{amsfonts}       % blackboard math symbols
\usepackage{amsmath}
\usepackage{amssymb}
\usepackage{amsthm}
\usepackage{authblk}
\usepackage{color}
\usepackage{bbm}
\usepackage{enumerate}
\usepackage{float}
\usepackage[margin=1in]{geometry}
\usepackage{graphicx}
\usepackage{framed}
\usepackage[round]{natbib}
\usepackage{subfig}
\usepackage{xifthen}
\usepackage{multirow}

\newcommand{\norm}[1]{\lVert#1\rVert}

\newcommand{\grad}{\nabla}
\newcommand{\param}{x}
\newcommand{\seq}{\param}
\newcommand{\fseq}{y}
\newcommand{\corrupt}{\wt{\fseq}}

\newcommand{\wt}{\widetilde}
\newcommand{\opt}{\param^{*}}

\newcommand{\eps}{\varepsilon}
\newcommand{\iter}{\kappa}
\newcommand{\cost}{\iota}

\newcommand{\E}{\mathbb{E}}
\newcommand{\pr}{\mathbb{P}}

\newcommand{\systemname}{SCAR}

\newtheorem{thm}{Theorem}[section]
\newtheorem{lemma}[thm]{Lemma}

\newtheorem*{thm*}{Theorem}
\newtheorem*{lemma*}{Lemma}
\newtheorem*{cor*}{Corollary}
\newtheorem*{prop*}{Proposition}
\newtheorem*{conjecture*}{Conjecture}

\theoremstyle{definition}

\newtheorem*{defn*}{Definition}
\theoremstyle{definition}

\theoremstyle{definition}
\newtheorem{ex}{Example}[section]
\theoremstyle{remark}
\newtheorem*{ex*}{Example}
\theoremstyle{definition}

\theoremstyle{definition}
\newtheorem*{assm*}{Assumption}
\theoremstyle{remark}
\newtheorem{remark}{Remark}[section]
\theoremstyle{remark}
\newtheorem*{remark*}{Remark}

\usepackage[colorlinks=true, linkcolor=blue]{hyperref}

\title{Fault Tolerance in Iterative-Convergent Machine Learning}

\begin{document}

\author[1,2]{Aurick Qiao}
\author[2]{Bryon Aragam}
\author[1]{Bingjing Zhang}
\author[1,2]{Eric P. Xing}

\affil[1]{Petuum Inc.}
\affil[2]{Carnegie Mellon University}

\maketitle

\vskip 0.3in

\begin{abstract}
% !TEX root = ./main.tex

Machine learning (ML) training algorithms often possess an inherent self-correcting behavior due to their iterative-convergent nature. Recent systems exploit this property to achieve adaptability and efficiency in unreliable computing environments by relaxing the consistency of execution and allowing calculation errors to be self-corrected during training. However, the behavior of such systems are only well understood for specific types of calculation errors, such as those caused by staleness, reduced precision, or asynchronicity, and for specific types of training algorithms, such as stochastic gradient descent. In this paper, we develop a general framework to quantify the effects of calculation errors on iterative-convergent algorithms and use this framework to design new strategies for checkpoint-based fault tolerance. Our framework yields a worst-case upper bound on the iteration cost of arbitrary perturbations to model parameters during training. Our system, \systemname{}, employs strategies which reduce the iteration cost upper bound due to perturbations incurred when recovering from checkpoints. We show that \systemname{} can reduce the iteration cost of partial failures by \( 78 \)\%--\( 95 \)\% when compared with traditional checkpoint-based fault tolerance across a variety of ML models and training algorithms. 
% \todo[Bryon]{Updated abstract, please review.}
\end{abstract}

% !TEX root = ./main.tex

\section{Introduction}
\label{sec:intro}

Distributed model training for machine learning (ML) is a workload which is typically long-running and resource-intensive. Throughout a job's lifetime, it is susceptible to hardware failures, performance fluctuations, and other uncertainties inherent to real-world cluster environments. For example, processes can be preempted by a cluster resource allocator~\citep{yarn,mesos}, parameter synchronization can be bottlenecked on a slow or congested network~\citep{li-ps-nips,poseidon}, and stragglers can severely impact overall job throughput~\citep{cipar-stragglers,harlap-stragglers}. Thus, developing new fault-tolerance strategies for modern ML systems is an important area of research.

ML-agnostic distributed systems approaches for addressing such problems often adopt strong consistency semantics. They aim to provide strong execution guarantees at a per-operation level (such as linearizability or serializability), but may also incur higher performance overhead. 
On the other hand, ML training is often tolerant to small calculation errors and may not require such strong consistency guarantees. This observation has been exploited by recent ML systems to overcome cluster unreliability and resource limitation issues. For example, bounded staleness consistency allows stale model parameters to be used, reducing the cost of synchronization and mitigating the effects of stragglers and/or congested networks~\citep{ho-ssp,cipar-stragglers,cui-staleness}. Training using quantized or low-precision floating point representations drastically reduces the overhead of computation and communication limitations~\citep{DBLP:journals/corr/CourbariauxBD14,DBLP:journals/corr/GuptaAGN15,hubara-qnn}. Lock-free execution eliminates the cost of blocking synchronization primitives, achieving higher throughput parallel training~\citep{Niu:2011:HLA:2986459.2986537,Dean:2012:LSD:2999134.2999271}. One notable exception to this trend is checkpoint-based fault tolerance, a common strategy in current ML systems for mitigating hardware failures \citep{199317,Wei:2015:MCC:2806777.2806778,Low:2012:DGF:2212351.2212354} which continues to enforce strong consistency semantics at a high cost of re-computing lost work.

\begin{figure}
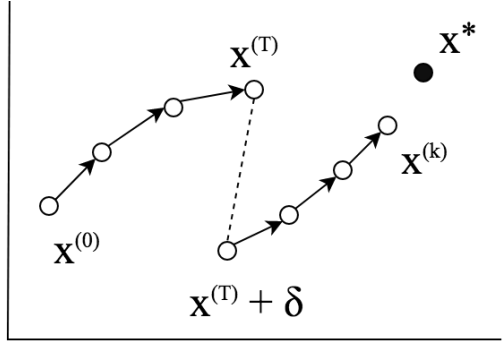

\centering{
\includegraphics[width=0.4\columnwidth]{img/{{self-correcting}}}}
\caption{The self-correcting behavior of iterative-convergent algorithms. Even though a calculation error results in an undesirable perturbation of \( \delta \) at iteration \( T \), the subsequent iterations still brings the solution closer to the optimum value of \( x^* \).}
\label{fig:self-correcting}
\end{figure}

This trend of relaxing consistency in ML systems relies on a fundamental trade-off, allowing the training algorithm to incur computation errors in order to gain more freedom to optimize execution and adapt to environmental faults. To preserve correctness, it relies on the \textit{self-correcting} behavior of iterative-convergent ML training algorithms. During each step, the training algorithm calculates updates based on the current values of model parameters, and then applies the updates to obtain a ``better'' set of model parameters. By iteratively performing this computation, the model parameters eventually converge to a set of optimal values. Small computation errors made during this procedure are eventually washed out by the successive iterative improvements (see Fig.~\ref{fig:self-correcting}). 

This self-correcting behavior of ML training suggests a general strategy for designing robust training systems for unreliable environments, as follows:
\begin{enumerate}[(A)]
\item The execution system allows certain environmental faults and/or resource limitations to manifest as calculation errors in model training. These errors can be conceptualized as \textit{perturbations} to the model parameters. \label{rule-A}
\item The perturbations are self-corrected by the model training algorithm, which incurs an extra number of iterations. We refer to this number of extra iterations as the \textit{iteration cost} of the perturbations. \label{rule-B}
\end{enumerate}
Unfortunately, this approach has only been used and analyzed individually for a few limited cases such as the ones listed above, while other opportunities still exist. 

\begin{figure}
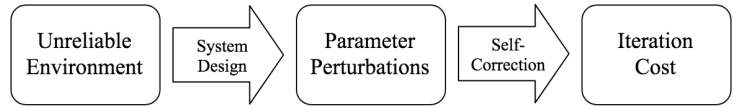

\centering{
\includegraphics[width=0.6\columnwidth]{img/{{fault-tolerance-framework}}}}
\caption{A framework for designing robust training systems by exploiting the self-correcting behavior of ML. First, through system design, resource instabilities and constraints in unreliable computing environments are allowed to manifest as perturbations in model parameters. Then, through the self-correcting behavior of ML, the perturbations are automatically corrected but incurs a cost in the number of iterations to convergence.}
\label{fig:fault-tolerance-framework}
\end{figure}

Motivated by this general strategy, in this paper we develop a framework for exploiting self-correction in ML systems in a way that is adaptive to generic perturbations whose cause or origin is unknown. It provides a theoretical foundation for understanding the self-correcting behavior of iterative-convergent model training as well as the tools needed by ML systems to take advantage of this behavior. Our main contributions are:
\begin{enumerate}
\item We quantify the impact of generic perturbations on iterative-convergent algorithms in terms of their iteration cost. Under reasonable convergence assumptions, we bound the iteration cost in terms of the sizes of these perturbations.
\item We propose new strategies for checkpoint-based fault tolerance in distributed model training. 
Partially recovering from checkpoints, combined with prioritizing checkpoints in a way that reduces the size of perturbations, can significantly reduce the iteration cost due to partial failures.
\item We design and implement \systemname{}, a parameter server system which employs our checkpoint strategies. We show that \systemname{} reduces the iteration cost of partial failures by \( 78 \)\%--\( 95 \)\% across a variety of popular ML models and training algorithms when compared with traditional checkpoint recovery.
\end{enumerate}

\section{Modeling Faults in~Iterative-Convergent Machine Learning}
\label{sec:modeling}

Most ML algorithms are iterative, i.e. model parameters are updated given a current estimate of the model parameters $\seq^{(k)}$ until convergence to some target parameter $\opt$. Such algorithms are commonly called \emph{iterative-convergent}. Examples include classical optimization algorithms such as gradient descent and stochastic gradient descent (SGD) as well as Monte Carlo methods such as Gibbs sampling and Metropolis-Hastings. These algorithms are the basic building blocks of models such as deep neural networks, matrix factorization, and latent Dirichlet allocation. In their most general form, such schemes can be written
\begin{align}
\label{eq:def:basic}
\seq^{(k+1)}
= f(\seq^{(k)}),
\quad
\seq^{(k)} \in\mathbb{R}^{d},
\end{align}
where \( \seq^{(k)} \) represents the model parameter values at iteration \( k \) and \( f \) updates the current state \( \seq^{(k)} \) to obtain the new state \( \seq^{(k+1)} \). In practice,  \( f \) is a known function that depends on the data available at time  \( k \). 

This model of iterative-convergent algorithms assumes that the current state \( \seq^{(k)} \) is stored persistently and losslessly in memory. In practice, modern distributed ML systems are subject to faults such as hardware failures, memory corruption, and performance fluctuations. Thus, it is unrealistic to assume that \( \seq^{(k)} \) can always be retrieved with perfect fidelity. To model this uncertainty, let $\delta_{k}$ be a random variable that represents an \emph{unknown} perturbation that corrupts the current state to produce a perturbed state $\seq^{(k)}+\delta_{k}$. 
We make no assumptions about the cause, size, or behavior of the perturbations $\delta_{k}$. More specifically, we assume the iterates obey the following scheme:
\begin{align}
\label{eq:gen:iter}
\begin{aligned}
\fseq^{(0)} &= \seq^{(0)} \\
\fseq^{(1)} &= f(\fseq^{(0)} +\delta_{0}) \\
% x^{(1)} &= f(y^{(0)})\\
&\,\,\,\vdots\\
\fseq^{(k+1)} &= f(\fseq^{(k)} +\delta_{k}) \\
% \seq^{(k+1)} &= f(y^{(k)})
\end{aligned}
\end{align}
In the absence of errors, ie. \( \delta_k = 0 \), we have $\fseq^{(k)}=\seq^{(k)}$, which reduces to the basic iterative scheme (\ref{eq:def:basic}). Moreover, since $\delta_{k}$ is arbitrary, this model allows for \emph{any} type of perturbation. In particular, perturbations may occur in every iteration or periodically according to some random process.

This setup captures many of the ways that cluster environment faults can be manifested as perturbations in distributed ML systems, and we give a few important examples below.

\begin{ex}[Reduced Precision]
\label{ex:faults:reduced-precision}
A simple practical example is using reduced precision floating/fixed point representations for storing parameter values. Suppose \( \wt\fseq^{(k)} \) is the reduced precision version of the exact parameter values \( \fseq^{(k)} \), then the algorithm suffers perturbations of \( \delta_k = \wt\fseq^{(k)} - \fseq^{(k)} \) at each iteration \( k \). If the representation has a \( p \)-bit mantissa, then the size of \( \delta_k \) is bounded by \( |\delta_k| < 2^{-(p-1)} |y^{(k)}| \)~\citep{Higham:2002:ASN:579525}.
\end{ex}

\begin{ex}[Bounded Staleness Consistency]
\label{ex:faults:staleness}
In SGD under the stale synchronous parallel (SSP) consistency model~\citep{ho-ssp}, gradients are computed in a data-parallel fashion where each of \( M \) machines may observe a stale version of the model parameters \( \wt\param_m^{(k)} \). Suppose \( \nabla(\wt\param_m^{(k)},\; D_m) \) are the gradients computed during iteration \( k \) using input data \( D_m \) at machine \( m \). If \( \nabla(\param^{(k)},\; D) \) is the true stochastic gradient at iteration \( k \), then the algorithm suffers a perturbation at iteration \( k + 1 \) of:
\[ \delta_{k+1} = \frac{1}{M}\sum_{m=1}^M \nabla(\wt\param_m^{(k)},\; D_m) - \nabla(\param^{(k)},\; D) \]
\end{ex}

\begin{ex}[Checkpoint-based Fault Tolerance]
\label{ex:faults:checkpoint}
In failure recovery from checkpoints, a copy of the entire job state is periodically saved to persistent storage, and is restored in the case of a failure. For distributed model training, the saved state includes the entire set of model parameters. Suppose a system experiences a failure at iteration \( T \), and recovers from the failure by restoring a full checkpoint of the model parameters taken at iteration \( C<T \). Then the algorithm suffers a perturbation at iteration \( T \) of $$\delta_T = \seq^{(T)} - \seq^{(C)}.$$
Although from the system's point of view the application is returned to an exact prior state, we can still view the act of checkpoint recovery as a perturbation to the model parameters.
\end{ex}

\begin{remark}
\label{rem:perturbations}
Reduced precision (Example \ref{ex:faults:reduced-precision}) and bounded staleness consistency (Example \ref{ex:faults:staleness}) have already been the focus of much attention in both the ML and systems communities~\citep{zipml,2018arXiv180711205J,Wei:2015:MCC:2806777.2806778,Dai:2015:HDM:2887007.2887019}. Although not typically studied within the explicit set-up of \eqref{eq:gen:iter}, these strategies generate perturbations which fit within our framework, and preserve the correctness of training by \textit{keeping the sizes of these perturbations small}. This is accomplished via bounded floating-point/fixed-point rounding errors for reduced precision and via a maximum staleness limit for bounded staleness consistency. In Section \ref{sec:fault-tolerance}, we apply the general set-up of \eqref{eq:gen:iter} to devise new strategies for checkpoint-based fault tolerance (Example \ref{ex:faults:checkpoint}). Our system, \systemname{}, applies the same principle of reducing the sizes of perturbations in order to reduce the overall iteration cost of machine failures.
\end{remark}

\begin{remark}
\label{rem:pgd}
The iteration in \eqref{eq:gen:iter} is closely related to \emph{perturbed gradient descent} (PGD) \citep{ge2015saddle,jin2017,du2017}. Formally, PGD is a special case of \eqref{eq:gen:iter}. The main difference lies in the motivation: \citet{jin2017} show that by choosing $\delta_{k}$ cleverly, it is possible to escape saddle points and guarantee that the iteration \eqref{eq:gen:iter} converges to a second-order stationary point. The key idea in PGD is to \emph{design} the perturbations $\delta_{k}$ to an advantage, which is in stark contrast to our set-up, in which we have no control over $\delta_{k}$. In the worst case, we allow $\delta_{k}$ to be chosen adversarially. 
\end{remark}

\section{Analysis}
\label{sec:analysis}

Suppose that an ML system has experienced perturbations $\delta_{1},\ldots,\delta_{T}$ up to the $T$th iteration. 
A (random) sequence $a_{k}$ is called \emph{$\eps$-optimal} if $\E\norm{a_{k}-\opt}<\eps$. The main question we seek to address in this section is the following: \emph{Given $\eps>0$, what is the ``cost'' in number of iterations for $\fseq^{(k)}$ to reach $\eps$-optimality compared to the unperturbed sequence $\seq^{(k)}$?} We write ``cost'' in quotations to emphasize that this number can be negative---for example, $\delta_{k}$ could randomly move $\fseq^{(k)}$ closer to $\opt$, or $\delta_{k}$ can be constructed in advance to improve convergence as in perturbed gradient descent (see Remark~\ref{rem:pgd}).

We call this quantity the \emph{iteration cost} of the perturbed sequence $\fseq^{(k)}$, introduced in Sec.~\ref{sec:intro}. Our goal in the present section is to bound the iteration cost, which will be formally defined next.

\subsection{Iteration cost}

\begin{figure*}
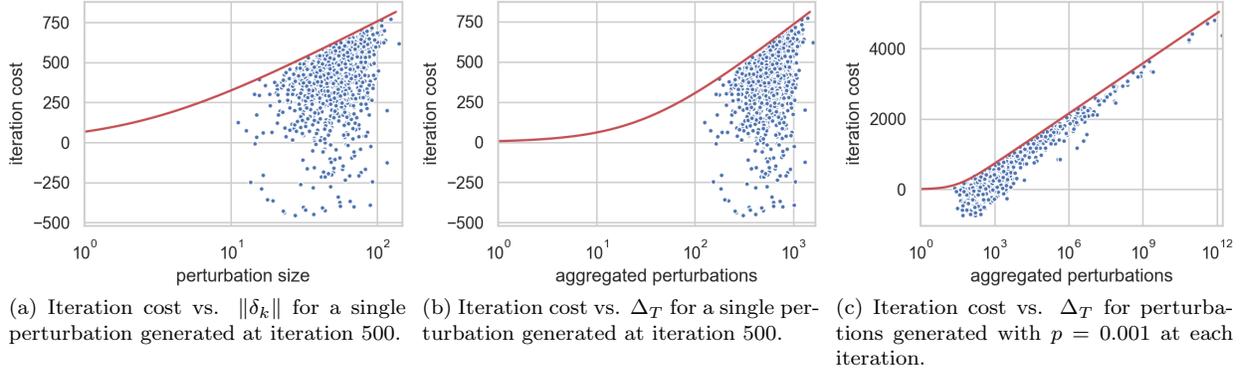

\centering
\subfloat[Iteration cost vs. \( \norm{\delta_k} \) for a single perturbation generated at iteration 500.]{\includegraphics[trim={12 12 12 12},clip,width=0.32\textwidth]{img/{{qp_single}}}}
\;
\subfloat[Iteration cost vs. \( \Delta_T \) for a single perturbation generated at iteration 500.]{\includegraphics[trim={12 12 12 12},clip,width=0.32\textwidth]{img/{{qp_single_sum}}}}
\;
\subfloat[Iteration cost vs. \( \Delta_T \) for perturbations generated with \( p = 0.001 \) at each iteration.]{\includegraphics[trim={12 12 12 12},clip,width=0.32\textwidth]{img/{{qp_rho}}}}
\caption{Illustrations of iteration costs using gradient descent on a simple 4-D quadratic program. Each plot consists of 1,000 trials with perturbation(s) randomly generated according to a normal distribution. The red line is the iteration cost bound according to Theorem~\ref{thm:main}. The value of \( c \) is determined empirically, and the value of \( \epsilon \) is set so that an unperturbed trial converges in roughly 1,000 iterations.}
\label{fig:qp-bounds}
\end{figure*}

In order to keep things simple, we assume that the unperturbed sequence satisfies
\begin{align}
\label{assm:seq:stepbound}
% \norm{\fseq^{(k+1)} - \opt} =
\norm{f(\seq^{(k)}) - \opt}
\le c\;\norm{\seq^{(k)} - \opt},
\quad 0<c<1,
\end{align}
\noindent
i.e. the iterates $\seq^{(k)}$ converge linearly. Although some algorithms (e.g. SGD) do not converge linearly, many of the most popular algorithms in practice do (e.g. gradient descent, proximal quasi-Newton, Metropolis-Hastings). This assumption is made purely for simplicity: We use \eqref{assm:seq:stepbound} as a baseline for comparison, and the analysis can be easily extended to more general schemes such as SGD if desired (see Example~\ref{ex:analysis:sgd}).

Formally, the iteration cost is defined as follows: Let $\iter(\fseq^{(k)},\;\eps)$ be a lower bound such that $m>\iter(\fseq^{(k)},\;\eps)$ implies $\E\norm{\fseq^{(m)}-\opt}<\eps$ (this may be $+\infty$ or negative).
Under \eqref{assm:seq:stepbound}, it is straightforward to derive a similar lower bound for the unperturbed sequence $\seq^{(k)}$ as $\iter(\seq^{(k)},\;\eps)=\log\big(\frac1\eps\norm{\seq^{(0)} - \opt}\big)/\log(1/c)$. This will be used as a baseline for comparison: The iteration cost for the perturbations $\delta_{k}$ is defined to be
\begin{align}
\label{eq:def:itercost}
\cost(\delta_{k},\;\eps)
:= \iter(\fseq^{(k)},\;\eps) - \iter(\seq^{(k)},\;\eps).
\end{align}
\noindent
Using the unperturbed sequence $\seq^{(k)}$ as a benchmark, $\cost(\delta_{k},\;\eps)$ bounds the additional number of iterations needed for the perturbed sequence $\fseq^{(k)}$ to reach $\eps$-optimality (where we bear in mind that this can be negative). Clearly, $\cost(\delta_{k},\;\eps)$ depends on the sequence $\delta_{k}$, and should be smaller whenever the $\delta_{k}$ are smaller. We seek a bound on $\cost(\delta_{k},\;\eps)$ that holds for \emph{arbitrary $\delta_{k}$}.

\begin{remark}
We use the criterion $\E\norm{\fseq^{(k)}-\opt}<\eps$ as an optimality criterion instead of directly bounding $\pr(\norm{\fseq^{(k)}-\opt}<\eps)$. 
This is commonly done \citep[e.g.][]{bottou2016} since bounds on $\E\norm{\fseq^{(k)}-\opt}$ imply bounds on the latter probability via standard concentration arguments \citep[see e.g.][]{rakhlin2012making}. These bounds will of course depend on the tail behavior of $\delta_{k}$.
\end{remark}

\subsection{Bounding the iteration cost}

To bound the iteration cost, we also require that the update $f$ satisfies a  convergence rate similar to \eqref{assm:seq:stepbound} for the perturbed data $\corrupt^{(k)}:=\fseq^{(k)}+\delta_{k}$:
\begin{align}
\label{assm:fseq:stepbound}
% \norm{\fseq^{(k+1)} - \opt} =
\E\norm{f(\corrupt^{(k)}) - \opt}
\le c\;\E\norm{\corrupt^{(k)} - \opt},
\quad 0<c<1.
\end{align}

\noindent
This simply says that wherever the algorithm is, on average, a single step according to $f$ will not move the iterates further from $\opt$. 

For example, consider gradient descent, which is arguably one of the simplest iterative schemes in ML. For gradient descent, we have $f(\corrupt^{(k)}) = \corrupt^{(k)} - \alpha\grad\ell(\corrupt^{(k)})$, where $\ell$ is the objective function that is being minimized. Then we have the following:

\begin{lemma}[Strongly convex objective]
\label{lem:gd:mainassumption}
Suppose the objective function $\ell$ is strongly convex. Then the gradient descent algorithm satisfies \eqref{assm:fseq:stepbound}. 
\end{lemma}
\vspace{-1em}
\begin{proof}
See Appendix \ref{appendix:lem:gd:mainassumption}.
\end{proof}
\vspace{-0.5em}

\noindent
Similar results hold for other optimization schemes such as proximal methods and Newton's method. 
In fact, the assumption of strong convexity can be substantially relaxed to include various nonconvex problems \citep{xu2017globally,attouch2010proximal}. See Example~\ref{ex:analysis:nonconvex}.

\noindent
Under \eqref{assm:seq:stepbound} and \eqref{assm:fseq:stepbound}, we have the following general bound on the iteration cost: 

\begin{thm}
\label{thm:main}
Assume $\E\norm{\delta_{k}}<\infty$ for $k\le T$ and $\delta_{k}=0$ for $k>T$. Under \eqref{assm:seq:stepbound} and \eqref{assm:fseq:stepbound}, we have for any $\eps>0$,
\begin{align}
\label{eq:bound:itercost}
\cost(\delta_{k},\eps)
\le \frac{\log\Big(1 + \frac{\Delta_{T}}{\norm{x^{(0)} - \opt}}\Big)}{\log(1/c)}
\end{align}
\noindent
where $\Delta_{T}:=\sum_{\ell=0}^{T}c^{-\ell}\E\norm{\delta_{\ell}}$.
\end{thm}
\vspace{-1em}
\begin{proof}
See Appendix ~\ref{appendix:thm:main}.
\end{proof}
\vspace{-0.5em}

\noindent
In fact, the bound \eqref{eq:bound:itercost} is tight in the following sense: As long as \eqref{assm:seq:stepbound} cannot be improved, there exists a deterministic sequence $\delta_{1},\ldots,\delta_{T}$ such that \eqref{eq:bound:itercost} holds with equality. Theorem~\ref{thm:main} is illustrated on a simple quadratic program (QP) in Figure~\ref{fig:qp-bounds}, which provides empirical evidence of the tightness of the bound. Additional empirical experiments are illustrated in Figure~\ref{fig:mlr-cost}.

The interesting part of the bound \eqref{eq:bound:itercost} is the ratio $\Delta_{T}/\norm{x^{(0)} - \opt}$, which is essentially a ratio between the aggregated cost of the perturbations and the ``badness'' of the initialization. For more intuition, re-write this ratio as 
\begin{align*}
\frac{\Delta_{T}}{\norm{x^{(0)} - \opt}}
% = \frac{c^{k}\Delta_{T}}{c^{k}\norm{x^{(0)} - \opt}}
= \frac{\sum_{\ell=0}^{T}c^{k-\ell}\E\norm{\delta_{\ell}}}{c^{k}\norm{x^{(0)} - \opt}}.
\end{align*}

\noindent
Up to constants, the denominator is just the error of the original sequence $\seq^{(k)}$ after $k$ iterations. The numerator is more interesting: It represents a time-discounted aggregate of the overall cost of each perturbation. Each perturbation $\delta_{\ell}$ is weighted by a discount factor $c^{k-\ell}$, which is larger for more recent perturbations (e.g. $\delta_{T}$) and smaller for older perturbations (e.g. $\delta_{0}$). Thus, the dominant quantity in \eqref{eq:bound:itercost} is a ratio between the re-weighted perturbations and the expected error from the original sequence. As expected, if the original sequence converges very quickly and the perturbations are large, the iteration cost increases proportionally.

Theorem~\ref{thm:main} also assumes that there are no perturbations after time $T$. The idea is that \emph{if} there are no more perturbations, \eqref{eq:bound:itercost} bounds the cost of the perturbations incurred so far. Of course, in practice, the system may experience faults after time $T$, in which case \eqref{eq:bound:itercost} can be adjusted to include the most recent fault. The difficulty in directly accounting for future perturbations lies in our assumption that the $\delta_{k}$ can be arbitrary: If future iterations can experience \emph{any} perturbation, it is clear that convergence cannot be guaranteed (e.g. consider $\delta_{k}=x-\fseq^{(k)}$ for some fixed $x\ne\opt$ and all $k>T$). Under some additional assumptions, something can be said about this case; see Example~\ref{ex:analysis:infinite} in the next section.

\subsection{Examples}

In this section, we discuss some examples where the bound \eqref{eq:bound:itercost} is applicable, along with some generalizations.

\begin{ex}[Convex optimization]
\label{ex:analysis:convex}
Lemma~\ref{lem:gd:mainassumption} implies that Theorem~\ref{thm:main} applies to ML systems that are based on minimizing a strongly convex objective. This includes many classical problems such as linear and logistic regression. 
\end{ex}

\begin{ex}[Nonconvex optimization]
\label{ex:analysis:nonconvex}
If the loss function $\ell$ is nonconvex, then Theorem~\ref{thm:main} still applies with some modifications. The assumptions~\eqref{assm:seq:stepbound} and~\eqref{assm:fseq:stepbound} can be verified using known results on nonconvex optimization \citep{xu2017globally,attouch2010proximal} under the so-called \emph{Kurdyka-\L ojasiewicz property}, from which the bound \eqref{eq:bound:itercost} follows directly. Trouble arises, however, when $\ell$ has multiple basins of attraction: A perturbation $\delta_{k}$ could ``push'' the perturbed iterate $\corrupt^{(k)}$ into a different basin, resulting in a limit point that is different from $\opt$. Theorem~\ref{thm:main} continues to hold as long as this can be avoided, i.e. the $\delta_{k}$ are not too large. We leave it to future work to study this case in more detail.
\end{ex}

\begin{ex}[Infinite perturbations]
\label{ex:analysis:infinite}
An interesting case occurs when $\delta_{k}\ne 0$ for all $k$. In other words, there is a possibility of a fault in \emph{every iteration}. For arbitrary $\delta_{k}$, it is clearly impossible to establish any kind of convergence result. In fact, suppose $\norm{\delta_{k}}\le\Delta$ for each $k$. Then there is an irreducible error of $(c/(1-c))\Delta$, meaning that we cannot hope to obtain an $\eps$-optimal solution for any $\eps<(c/(1-c))\Delta$. This helps to explain why we focus on the nontrivial case with $\delta_{k}=0$ for $k>T$ in Theorem~\ref{thm:main}. One setting in which the analysis with infinite perturbations is nontrivial is when $\Delta$ is known to be small, e.g. when using reduced precision as in Example~\ref{ex:faults:reduced-precision}. This setting can be analyzed by setting \( \Delta \ge 2^{-(p-1)}\norm{\param^{(k)}} \) for all \( k \). For details, see Appendix~\ref{appendix:infinite}. 

\end{ex}

\begin{ex}[SGD]
\label{ex:analysis:sgd}
The assumption \eqref{assm:fseq:stepbound} does \emph{not} hold for SGD, which has a sublinear convergence rate in general. Nonetheless, it is straightforward to extend our framework to sublinear algorithms, with the caveat that analogous bounds on the iteration cost become more complicated. In fact, it is not hard to see from our proof how to do this: Lemma~\ref{lem:periodic:basic} in the Appendix establishes the following useful general inequality
\begin{align*}
\E\norm{\fseq^{(k+1)} - \opt}
&\le c^{k+1}\big[\norm{x^{(0)} - \opt} + \Delta_{T}\big].
\end{align*}
Evidently, the factor of $c$ governs how quickly $\Delta_{T}$ (i.e. the cost incurred by perturbations) gets washed out as $k$ increases. For algorithms that converge sublinearly such as SGD, this effect will also be sublinear, but still tend to zero as long as the perturbations are not too large (see Appendix~\ref{appendix:sgd} for a brief discussion). This is further corroborated by the empirical experiments in Section~\ref{sec:exp}, where we show that the strategies for checkpoint-based fault tolerance proposed in the next section are successful on SGD as well as other optimization schemes such as alternating least squares.
\end{ex}

\section{Strategies for~Checkpoint-Based Fault Tolerance}
\label{sec:fault-tolerance}

As an application of our iteration cost bounds, we study new strategies for checkpoint-based fault tolerance, by which stateful computations are made resilient to hardware failures by periodically checkpointing the entire program state. Whenever such a failure occurs, the most recent checkpoint is restored and computation is resumed (Example~\ref{ex:faults:checkpoint}). The total running time \( T \) of the system can be modeled as~\citep{daly-optimum-checkpoint}:
\[ T = T_{\textup{solve}} + T_{\textup{dump}} + T_{\textup{rework}} + T_{\textup{restart}} \]
where \( T_{\textup{solve}} \) is the normal runtime of the program without any failures, \( T_{\textup{dump}} \) is the time taken saving a checkpoint, \( T_{\textup{rework}} \) is the time spent repeating lost work due to restoring to a previous checkpoint, and \( T_{\textup{restart}} \) is the time spent restoring a checkpoint. For iterative ML training, \( T_{\textup{rework}} = \cost(\delta_{k},\;\eps) \cdot T_{\textup{iter}} \), where \( \cost(\delta_{k},\;\eps) \) is the iteration cost of the failure and \( T_{\textup{iter}} \) is the time taken per iteration of the training algorithm. We focus our strategies primarily on reducing \( T_{\textup{rework}} \), since \( T_{\textup{solve}} \) and \( T_{\textup{iter}} \) are constant regardless of checkpoints and failures, and \( T_{\textup{dump}} \) and \( T_{\textup{restart}} \) are typically small fractions of \( T_{\textup{iter}} \) in our experiments. 

Using the traditional checkpoint-based fault tolerance mechanism, the entire program state is saved during each checkpoint, restored after a failure, and all computation since the previous checkpoint repeated. Thus, \( T_{\textup{rework}} \) is the total amount of time between the previous checkpoint and the failure, spanning the computation which must be repeated. This process maximizes the consistency of recovery by restoring the system to an exact state it was in during the past, but can incur a high rework overhead if the checkpoint interval is long.

For iterative-convergent ML, however, we can exploit its self-correcting behavior to reduce \( T_{\textup{rework}} \). In particular, we can give up the consistency of checkpoint-recovery, and design a system which tries to reduce the size of the perturbation \( \norm{\delta_T} \) incurred upon failure. By doing so, Theorem~\ref{thm:main} shows that the iteration cost bound is also reduced, lowering the worst case iteration cost and thus reducing \( T_{\textup{rework}} \).

Our system, \textit{\systemname{}},\footnote{SCAR stands for Self-Correcting Algorithm Recovery.} implements two strategies which reduce \( \norm{\delta_T} \) compared to traditional checkpoint recovery: (1) Partial recovery, and (2) Prioritized checkpoints. \systemname{} extends the popular parameter server (PS) architecture for distributed model training~\citep{ho-ssp,li-ps-nips,186214}---the model parameters are partitioned across a number of PS nodes, which are accessed by worker nodes. We assume that during a failure, any number of PS nodes can go down, causing the loss of their partitions of the model parameters. We present these strategies and the design of \systemname{} below, and show evaluation of \systemname{} in Section ~\ref{sec:exp}.

\subsection{Partial Recovery}
\label{sec:partial-recovery}

Our first strategy is to only recover (i.e. from a previous checkpoint) the part of the model parameters which are lost due to the failure. Since the model parameters are partitioned across several PS nodes, a partial failure of PS nodes should only cause a partial loss of model parameters. Mathematically, the partial recovery strategy should result in a smaller perturbation to the model parameters and, according to Theorem \ref{thm:main}, incur a smaller iteration cost.

Suppose that
a fully-consistent checkpoint
is taken after iteration \( C \)
, and a failure occurs during iteration \( T > C \) which triggers checkpoint recovery.
\begin{thm}
\label{thm:partial-recovery}
Let \( \delta \) be the perturbation incurred by full checkpoint recovery, and \( \delta' \) be the perturbation incurred by partial checkpoint recovery, then \( \norm{\delta'} < \norm{\delta} \).
\end{thm}
\vspace{-1em}
\begin{proof}
See Appendix ~\ref{appendix:thm:partial-recovery}.
\end{proof}
\vspace{-0.5em}
Furthermore, the size of the perturbation should also be related to the fraction of model parameters which are lost---losing fewer model parameters should generate a smaller perturbation. To establish this relationship, we will assume that parameters are partitioned uniformly at random across the PS nodes, and so a random subset of parameters will be lost. This assumption is reasonable as the partitioning scheme is typically within the control of the PS system, which can choose a random partitioning.

\begin{thm}
\label{thm:partial-recovery-exp} 
Suppose that a failure causes the loss of a fraction \( 0 < p \le 1 \) of all model parameters chosen uniformly at random. Let \( \delta \) be the perturbation incurred by full checkpoint recovery, and \( \delta' \) be the perturbation incurred by partial checkpoint recovery, then \( \mathbb{E}||\delta'||^2 = p||\delta||^2 \).
\end{thm}
\vspace{-1em}
\begin{proof}
See Appendix ~\ref{appendix:thm:partial-recovery-exp}.
\end{proof}
\vspace{-0.5em}
Thus, the expected size of perturbation incurred by partially restoring from a checkpoint decreases as the fraction of parameters lost decreases.

\subsection{Priority Checkpoint}
\label{sec:priority-checkpoint}

With the partial recovery strategy, we have shown that relaxing the consistency of checkpoint recovery can reduce the size of perturbations (i.e. $\delta_{k}$) experienced by the training algorithm due to a failure, and thus reduce the iteration cost. In this section, we further consider relaxing the consistency of saving checkpoints by taking more frequent, partial checkpoints.

Rather than saving all parameters every \( C \) iterations, consider saving a fraction \( r < 1 \) of the parameters every \( rC \) iterations. A \textit{running checkpoint} is kept in persistent storage, which is initialized to the initial parameter values \( \param^{(0)} \) and updated each time a partial checkpoint is saved. At a given time, this checkpoint may consist of a mix of parameters saved during different iterations, and the choice of which subset of parameters to checkpoint can be controlled via system design. This strategy enables, e.g., \textit{prioritization} of which parameters are saved during each checkpoint so as to prioritize saving parameters that will minimize the size of the perturbation caused by a failure. To do this, we consider a simple heuristic: Save the parameters which have changed the most since they were previously saved.

The checkpoint period \( rC \) is chosen so that the number of parameters saved every \( C \) iterations remains roughly constant across different values of \( r \). As a result the prioritized checkpoint strategy writes the same amount of data per constant number of iterations to persistent storage as the full checkpoint strategy, while having more frequent opportunities to prioritize and save parameters to the running checkpoint. We evaluate the system overhead implications of this scheme in Section \ref{sec:system-overhead}.

\subsection{\systemname{} Architecture and Implementation}
\label{sec:system}

We implement our system, \systemname{}, using these two checkpoint-based fault tolerance strategies. \systemname{} is implemented as a PS architecture---the parameters of the ML model are randomly partitioned across PS nodes, while the input data is partitioned across worker nodes. During each iteration, the workers read values from the PS nodes, compute updates using their local input data, and send the updates to the PS nodes to be applied.

Figure~\ref{fig:system-architecture} illustrates the architecture of \systemname{}. A \textit{fault tolerance controller} runs as a separate service and consists of (1) a \textit{checkpoint coordinator} responsible for coordinating periodic checkpoints at a fixed time interval, and (2) a \textit{recovery coordinator} responsible for coordinating the failure recovery process whenever a failure is detected. The detection of failures is performed by a \textit{failure detector} service, which can leverage heartbeating mechanisms in existing systems for distributed consensus such as ZooKeeper~\citep{Hunt:2010:ZWC:1855840.1855851}. Checkpoints are saved to shared persistent storage, such as distributed filesystems like NFS~\citep{Sandberg:1988:DIS:59309.59338}, CephFS~\citep{Weil:2006:CSH:1298455.1298485}, or distributed databases like Cassandra~\citep{Lakshman:2010:CDS:1773912.1773922}. To speed up distance calculations between the current and previously saved parameters, each PS node keeps an in-memory cache of the current checkpoint, which is updated whenever a new partial checkpoint is saved.

\begin{figure}
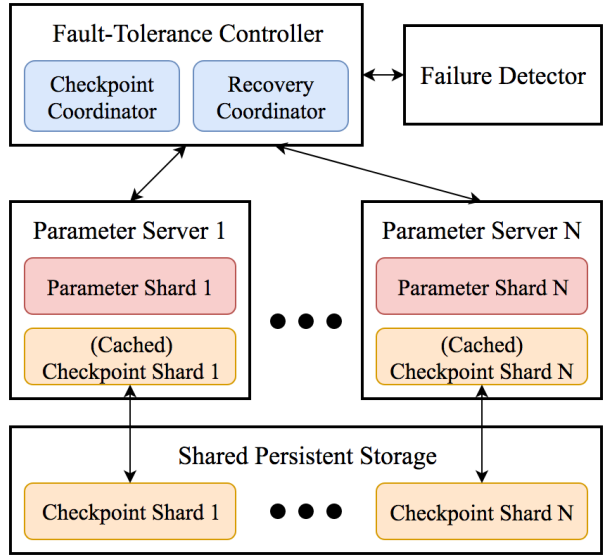

\centering{
\includegraphics[width=0.5\columnwidth]{img/{{system-architecture}}}}
\caption{\systemname{} system architecture for partial recovery and prioritized checkpoints in distributed model training. Details in Section \ref{sec:system}.}
\label{fig:system-architecture}
\end{figure}

When a checkpoint is triggered:
\begin{enumerate}
\item The checkpoint coordinator sends a message to each PS node, which computes the distance of each of its parameters from their previously saved values in the running checkpoint using its in-memory cache.
\item Each PS node sends its model parameter IDs and computed distances to the checkpoint coordinator.
\item Upon receipt of the computed distances from all PS nodes, the checkpoint coordinator selects the fraction $r$ of parameters with the largest distances, and sends their IDs back to their corresponding PS nodes.
\item Upon receipt of the parameter IDs, each PS node updates its in-memory cache, and saves those parameters to the shared persistent storage.
\end{enumerate}
During step 4, the training algorithm can be resumed as soon as the in-memory caches have been updated, while output to the shared persistent storage happens asynchronously in the background. Thus, the checkpointing overhead \( T_{\textup{dump}} \) in \systemname{} is just the time needed for prioritizing parameters and updating the in-memory cache.

When a failure is detected:
\begin{enumerate}
\item The failure detector notifies the recovery coordinator, which determines how the parameters belonging to the failed PS nodes should be re-partitioned.
\item The recovery coordinator partitions and sends the failed parameter IDs to the remaining PS ndoes, which re-load the parameters from the current running checkpoint in shared persistent storage.
\end{enumerate}

\systemname{} is implemented using C++ and leverages an existing elastic ML framework~\citep{216041}, which provides mechanisms for transparently re-routing requests from workers away from failed PS nodes, as well as for new PS nodes to join the active training job, replacing the old failed PS nodes.

\section{Experiments}
\label{sec:exp}

With our evaluation, we wish to (1) illustrate our iteration cost bounds for different types of perturbations using practical ML models, (2) empirically measure the iteration costs of a variety of models under the partial recovery and prioritized checkpoint strategies in \systemname{}, and (3) show that \systemname{} has low performance overhead.

\subsection{Models and Datasets}
\label{sec:models}

We use several popular models as examples for our analysis and checkpoint strategies. We describe their training algorithms, datasets, and parameter partitioning schemes below, and refer to Appendix \ref{appendix:experiments} for more details.

\paragraph{Multinomial Logistic Regression (MLR).} We use the standard stochastic gradient descent approach to minimize the multi-logit loss. The model parameters are an \( M \times N \) matrix of real numbers, where \( M \) is the dimensionality of the data, and \( N \) is the number of output classes. When distributed, the rows of the parameter matrix are randomly partitioned. We train MLR on the MNIST~\citep{726791} and CoverType~\citep{Dua:2017} datasets.

\paragraph{Matrix Factorization (MF).} We use the standard alternating least squares (ALS) approach to minimize the objective function. The model parameters are matrices \( L \in \mathbb{R}^{m \times p} \) and \( R \in \mathbb{R}^{p \times n} \). When distributed, the rows of \( L \) and the columns of \( R \) are randomly partitioned. We train MF on the MovieLens~\citep{Harper:2015:MDH:2866565.2827872} and Jester~\citep{Goldberg:2001:ECT:593963.594023} datasets.

\paragraph{Latent Dirichlet Allocation (LDA).} We use the standard collapsed Gibbs sampling~\citep{10.2307/2290921} approach to learn the model parameters, which are the document-topic and word-topic distributions. We use a scaled total variation between document-topic distributions as the norm for computing distances between parameters. When distributed, the document-topic distributions are randomly partitioned across nodes. We do not consider loss of word-topic distributions because they can be re-generated from the latent token-topic assignments. More details on this setup are in Appendix \ref{appendix:experiments}. We train LDA on the 20 Newsgroups~\citep{Lang95} and Reuters~\citep{Lewis:2004:RNB:1005332.1005345} datasets.

\paragraph{Convolutional Neural Network (CNN).} We train a network consisting of 2 convolution layers with ReLU activations~\citep{Nair:2010:RLU:3104322.3104425} and max pooling followed by 3 fully-connected layers with ReLU activations using Adam~\citep{DBLP:journals/corr/KingmaB14}. Because of the structure in neural network models, we consider two different partitioning strategies: 1) In \textit{by-layer} partitioning, we assume that the layers of the network are randomly partitioned across nodes; and 2) In \textit{by-shard} partitioning, we further divide each layer's parameters into shards, and all shards are randomly partitioned across nodes. We train this CNN on the MNIST~\citep{726791} dataset.

\subsection{Iteration Cost Bounds}

\begin{figure}
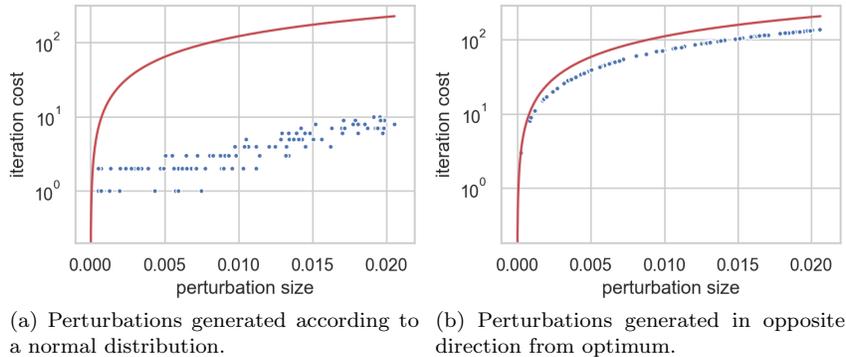

\centering
\subfloat[Perturbations generated according to a normal distribution.]{\includegraphics[trim={12 12 12 12},clip,width=0.33\textwidth]{img/{{cost_single}}}}
\;
\subfloat[Perturbations generated in opposite direction from optimum.]{\includegraphics[trim={12 12 12 12},clip,width=0.33\textwidth]{img/{{cost_single_adv}}}}
\;
\caption{Iteration costs of MLR on MNIST for (a) random perturbations and (b) adversarial perturbations. In each trial, a single perturbation is generated at iteration 50. The red line is the upper bound according to Theorem~\ref{thm:main}. The value of \( c \) is determined empirically, and the value of \( \epsilon \) is set so that an unperturbed trial converges in roughly 100 iterations.}
\label{fig:mlr-cost}
\end{figure}

\begin{figure}
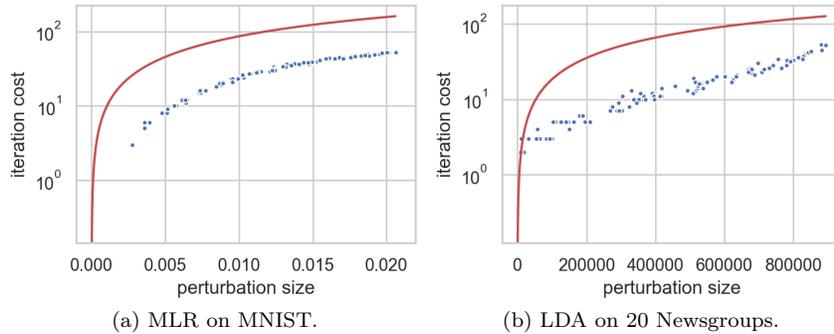

\centering
\subfloat[MLR on MNIST.]{\includegraphics[trim={12 12 12 12},clip,width=0.33\textwidth]{img/{{cost_single_fail}}}}
\;
\subfloat[LDA on 20 Newsgroups.]{\includegraphics[trim={12 12 12 12},clip,width=0.33\textwidth]{img/{{cost_single_20news}}}}
\caption{Perturbations are generated by resetting a random fraction of parameters back to their initial values, for both (a) MLR and (b) LDA. Other settings are the same as Figure \ref{fig:mlr-cost}.}
\label{fig:fail-cost}
\end{figure}

To illustrate the behavior of the iteration cost and to verify Theorem~\ref{thm:main} for different types of models and perturbations, we train MLR and LDA and generate a perturbation according to one of three types: random, adversarial, and resets.

For random perturbations (Figure \ref{fig:mlr-cost}(a)), the iteration cost bound is a loose upper bound on the actual iteration cost. This is in contrast to the simpler quadratic program (QP) experiments shown in Figure \ref{fig:qp-bounds}, in which the bound is relatively tight. On the other hand, we also do not observe any perturbations resulting in a negative iteration cost as for QP. This experiment shows that for MLR, a perturbation in a random direction is unlikely to greatly impact the total number of iterations to convergence.

We run a second experiment in which we generate ``adversarial'' perturbations opposite the direction of convergence (Figure \ref{fig:mlr-cost}(b)). In this case, we see that our bound is much closer to the actual iteration costs, indicating that it is still a tight worst-case upper bound on the iteration cost for MLR.

While Figure \ref{fig:mlr-cost} shows the iteration costs for synthetically generated perturbations, Figure \ref{fig:fail-cost} generates more realistic perturbations for both MLR and LDA. We generate perturbations by resetting a random subset of model parameters back to their initial values. This scheme simulates the type of perturbations the training algorithm would observe in the partial recovery scenario described in Section \ref{sec:partial-recovery}. In this case, we see that the behavior of actual iteration costs is closer to the scenario with adversarial perturbations, although not quite as costly.

\subsection{Partial Recovery}

\begin{figure*}
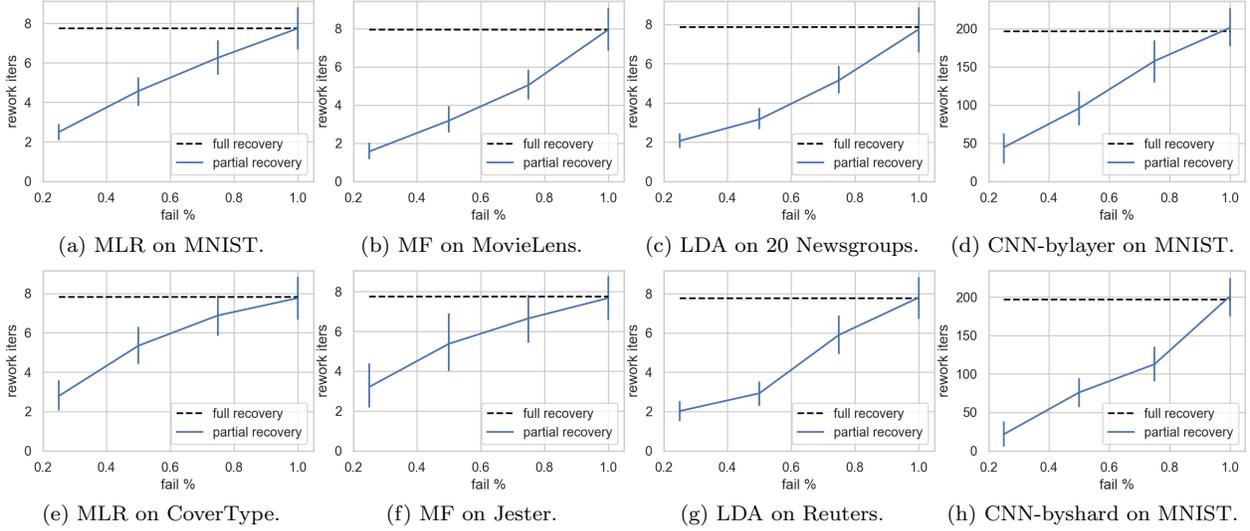

\centering
\subfloat[MLR on MNIST.]{\includegraphics[trim={12 12 12 12},clip,width=0.25\textwidth]{img/{{byfail-sampled.mlr.mnist}}}}
\subfloat[MF on MovieLens.]{\includegraphics[trim={12 12 12 12},clip,width=0.25\textwidth]{img/{{byfail-sampled.mf.movielens}}}}
\subfloat[LDA on 20 Newsgroups.]{\includegraphics[trim={12 12 12 12},clip,width=0.25\textwidth]{img/{{byfail-sampled.lda_tv_randinit.20news}}}}
\subfloat[CNN-bylayer on MNIST.]{\includegraphics[trim={12 12 12 12},clip,width=0.25\textwidth]{img/{{byfail-sampled.cnn_bylayer.mnist-adam}}}}
\\[-8pt]
\subfloat[MLR on CoverType.]{\includegraphics[trim={12 12 12 12},clip,width=0.25\textwidth]{img/{{byfail-sampled.mlr.covtype}}}}
\subfloat[MF on Jester.]{\includegraphics[trim={12 12 12 12},clip,width=0.25\textwidth]{img/{{byfail-sampled.mf_randinit.jester3}}}}
\subfloat[LDA on Reuters.]{\includegraphics[trim={12 12 12 12},clip,width=0.25\textwidth]{img/{{byfail-sampled.lda_tv_randinit.reuters}}}}
\subfloat[CNN-byshard on MNIST.]{\includegraphics[trim={12 12 12 12},clip,width=0.25\textwidth]{img/{{byfail-sampled.cnn_byshard.mnist-adam}}}}
\\[-8pt]
\caption{Partial vs. full recovery for a variety of models and datasets, where the set of failed parameters are selected uniformly at random. The $x$-axis shows the fraction of failed parameters, and the $y$-axis shows the number of rework iterations. The error bars indicate 95\% confidence intervals, calculated by repeating each trial 100 times.}
\label{fig:partial_recovery}
\end{figure*}

To empirically characterize the behavior of partial recovery from checkpoints, we simulate failures of varying fractions of model parameters for the MLR, MF, LDA, and CNN models. We compare the iteration costs incurred by full recovery with the iteration costs incurred by partial recovery. For each model, we sample the failure iteration from a geometric distribution, which causes the loss of a subset of model parameters chosen uniformly at random.

Fig.~\ref{fig:partial_recovery} shows the results. For all models and datasets, we see the average iteration cost incurred by partial recovery decreases as the failure fraction decreases. Meanwhile, the average iteration cost incurred by full recovery remains constant at its maximum value, since all parameters are loaded from the checkpoint regardless of which are actually lost.

Across all models and datasets tested, \systemname{} with partial recovery reduces the iteration cost by \( 12 \)\%--\( 42 \)\% for \( 3/4 \) failures, \( 31 \)\%--\( 62 \)\% for \( 1/2 \) failures, and \( 59 \)\%--\( 89 \)\% for \( 1/4 \) failures.

\subsection{Priority Checkpoint}

\begin{figure*}
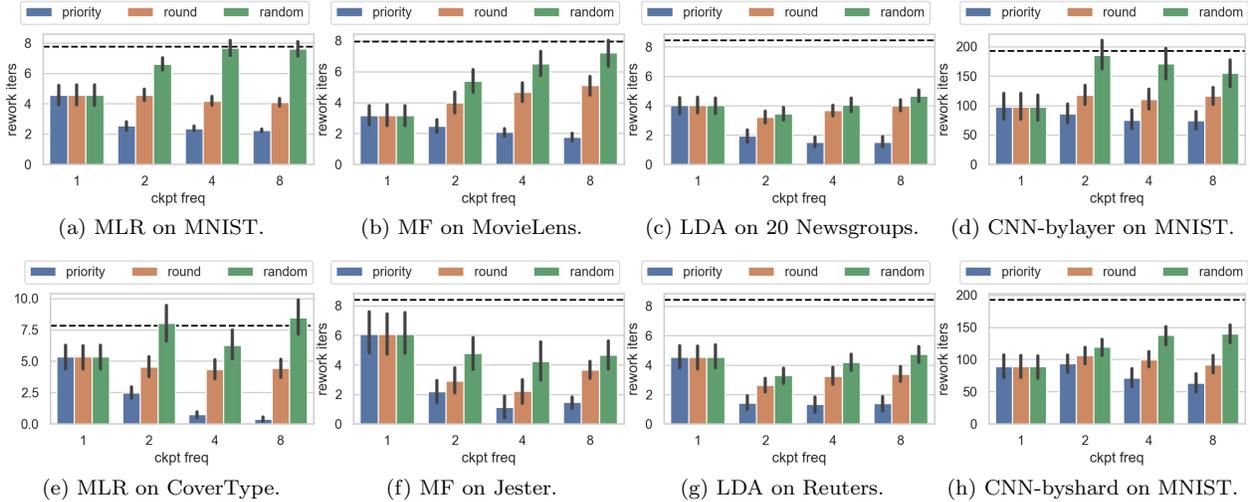

\centering
\subfloat[MLR on MNIST.]{\includegraphics[trim={12 12 12 20},clip,width=0.25\textwidth]{img/{{byckpt.mlr.mnist}}}}
\subfloat[MF on MovieLens.]{\includegraphics[trim={12 12 12 20},clip,width=0.25\textwidth]{img/{{byckpt.mf.movielens}}}}
\subfloat[LDA on 20 Newsgroups.]{\includegraphics[trim={12 12 12 20},clip,width=0.25\textwidth]{img/{{byckpt.lda_tv_randinit.20news}}}}
\subfloat[CNN-bylayer on MNIST.]{\includegraphics[trim={12 12 12 20},clip,width=0.25\textwidth]{img/{{byckpt.cnn_bylayer.mnist-adam}}}}
\\[-8pt]
\subfloat[MLR on CoverType.]{\includegraphics[trim={12 12 12 20},clip,width=0.25\textwidth]{img/{{byckpt.mlr.covtype}}}}
\subfloat[MF on Jester.]{\includegraphics[trim={12 12 12 20},clip,width=0.25\textwidth]{img/{{byckpt.mf_randinit.jester3}}}}
\subfloat[LDA on Reuters.]{\includegraphics[trim={12 12 12 20},clip,width=0.25\textwidth]{img/{{byckpt.lda_tv_randinit.reuters}}}}
\subfloat[CNN-byshard on MNIST.]{\includegraphics[trim={12 12 12 20},clip,width=0.25\textwidth]{img/{{byckpt.cnn_byshard.mnist-adam}}}}
\\[-8pt]
\caption{Prioritized checkpoint experiments comparing between the random, round-robin, and priority strategies. The $x$-axis indicated checkpoint frequency relative to full checkpoints, where \( 1 \) indicates full checkpoints, \( 2 \) indicates \( 1/2 \) checkpoints at \( 2\times \) frequency, etc., and the $y$-axis shows the number of rework iterations. The error bars indicate 95\% confidence intervals, calculated by repeating each trial 100 times, and the dashed black line represents the rework cost of a full checkpoint.}
\label{fig:partial_checkpoint}
\end{figure*}

In this section, we evaluate the effectiveness of our priority checkpoint strategy for the MLR, MF, LDA, and CNN models. We compare the iteration costs incurred by different fractions of partial checkpoints, while keeping constant the number of parameters saved per constant number of iterations, as described in Section \ref{sec:priority-checkpoint}. As before, we sample the failure iteration from a geometric distribution. In this experiment keep the fraction of lost parameters fixed at \( 1/2 \).

To gauge the effectiveness of prioritization, we compare between the following strategies:
\begin{enumerate}
\item \texttt{priority}: Parameters saved to checkpoint are selected based on the prioritization described in Section \ref{sec:priority-checkpoint}.
\item \texttt{round}: Parameters saved to checkpoint are selected in a round-robin manner.
\item \texttt{random}: Parameters saved to checkpoint are selected uniformly at random.
\end{enumerate}
Fig.~\ref{fig:partial_checkpoint} shows the results. For all models and datasets, we see the \texttt{priority} strategy results in decreasing iteration costs when the fraction of each checkpoint decreases (and frequency of checkpoints increases). On the other hand, the \texttt{round} strategy either reduces or increases the iteration cost depending on the model and dataset, while the \texttt{random} strategy nearly always increases the iteration cost.

Across all models and datasets tested, combining partial recovery with prioritized \( 1/8 \)th checkpoints at 8\(\times \) frequency reduces the iteration cost of losing \( 1/2 \) of all model parameters by \( 78 \)\%--\( 95 \)\% when compared with traditional checkpoint recovery.

\subsection{System Overhead}
\label{sec:system-overhead}

Lastly, we evaluate the system overhead of \systemname{} by training LDA on a 12GB subset of the ClueWeb12 dataset \citep{clueweb12}. The dataset contains 480K documents and 2B tokens, and the number of topics is set to 1K. We use four AWS i3.2xlarge instances, each with 1.9TB NVMe SSDs. For persistent storage, we install a CephFS on these machines.

We trigger a failure of \( 1/2 \) of model parameters during the \( 7 \)th iteration, and compare SCAR using \( 1/4 \)th checkpoints with the traditional full checkpoint-recovery mechanism. Figure \ref{fig:cf-conv} shows the convergence plots. Using \systemname{}, the same likelihood value is reached roughly 3 iterations sooner than using traditional checkpoint-recovery. Each iteration takes \( \approx 243 \) seconds, and \systemname{} spends an extra \( \approx 13 \) seconds for checkpointing after each iteration. Overall, performance overhead of checkpointing in \systemname{} is small in comparison to the run-time of the training job, and results in a net reduction of \( T_{\textup{rework}} \approx 6 \) min in rework time incurred due to the failure. In dynamic-resource shared computing environments, this extra time can be leveraged by a scheduler to make more fine-grained resource allocation decisions between competing jobs.

\begin{figure}
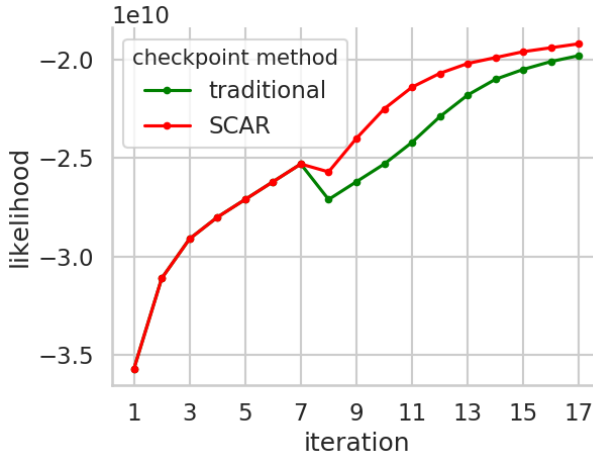

\centering{
\includegraphics[width=0.5\columnwidth]{img/{{lda}}}}
\caption{\systemname{} vs traditional checkpoint-recovery for LDA on ClueWeb. Using \systemname{}, we save \( 1/4 \) of model parameters every iteration. Using the traditional method, we save all model parameters every 4 iterations. SCAR reaches the same likelihood value roughly 3 iterations sooner, saving \( \approx 6 \) min.}
\label{fig:cf-conv}
\end{figure}

\section{Discussion and Related Work}
\label{sec:disc}

The model \eqref{eq:gen:iter} is closely related to several models in the literature. As discussed in Remark~\ref{rem:pgd}, perturbed gradient descent \citep{jin2017} is (formally) a special case of \eqref{eq:gen:iter}, however, the motivations are quite different. \citet{mania2015} and \citet{elgamal2017} also consider a model similar to \eqref{eq:gen:iter}, however, perturbations are only added to the gradients. 
Neither of these works consider perturbations in both the gradients and the current state $\seq^{(k)}$, as we do.

A related body of work is distributed training under Byzantine faults \citep{blanchard2017,chen2017,damaskinos2018,guerraoui2018hidden}, where a proportion of machines may act adversarially. Byzantine failures are one of the most general assumptions on failures, and thus a Byzantine fault-tolerant training system is naturally tolerant to many other types of faults and perturbations. 
However, perturbations to parameters during training are not always Byzantine, and can often be controlled via system implementations, such as bounded staleness consistency models, or partial recovery and prioritized checkpointing as in the present work. 

In existing distributed ML systems, the fault tolerance problem is approached from an ML-agnostic perspective. For example, TensorFlow~\citep{199317} offers recovery from periodic checkpoints, while the parameter server of \citet{186214} offers live replication of parameter values. Proteus~\citep{Harlap:2017:PAM:3064176.3064182} proposes an approach for fault-tolerance on transient machines by using more reliable machines for active backup of program state. In comparison, our system takes advantage of the self-correcting nature of ML, offering lower iteration cost compared with traditional checkpoint-restart, and without the performance overhead of live replication or storing parameter state on designated reliable machines.

\section{Conclusions and Future Work}
\label{sec:conc}

The self-correcting behavior of ML forms the basis of system techniques that allow model training to achieve adaptability and efficiency in unreliable and resource-limited environments. In this paper, we outlined a general approach to design such systems by reducing the sizes of perturbations to model parameters. We derived an upper bound on the iteration cost of perturbations which can guide the design of new systems. We then proposed and implemented new strategies for checkpoint-based fault tolerance in our system \systemname{}. We showed that \systemname{} is able to reduce the iteration cost of failures by an order of magnitude or more when compared to traditional checkpoint-based fault tolerance.

As for future work, we have already observed that our main assumptions \eqref{assm:seq:stepbound} and \eqref{assm:fseq:stepbound} can be relaxed, however, it remains to study these generalizations in more detail. In particular, it would be interesting to study the case of nonconvex $\ell$ (Example~\ref{ex:analysis:nonconvex}) more carefully in addition to sublinear schemes such as SGD (Example~\ref{ex:analysis:sgd}). Furthermore, we have avoided making assumptions on the perturbations $\delta_{k}$, however, by imposing additional assumptions on the frequency or size of these perturbations, one could derive tighter upper bounds on the iteration cost.

On the systems side, 
there exists opportunities for systems to more directly utilize Theorem \ref{thm:main}. By approximating \( c \) and \( \norm{x^{(0)} - \opt} \), we may obtain a predictive model which can be evaluated on-the-fly to inform decisions made by a system during run-time. Furthermore, it would be interesting to see how the strategies used in \systemname{} can be applied to non-PS architectures such as all-reduce, which has recently become popular for distributed training \citep{DBLP:journals/corr/abs-1802-05799}.

\bibliographystyle{plainnat}
\bibliography{references}

\begin{thebibliography}{54}
\providecommand{\natexlab}[1]{#1}
\providecommand{\url}[1]{\texttt{#1}}
\expandafter\ifx\csname urlstyle\endcsname\relax
  \providecommand{\doi}[1]{doi: #1}\else
  \providecommand{\doi}{doi: \begingroup \urlstyle{rm}\Url}\fi

\bibitem[Abadi et~al.(2016)Abadi, Barham, Chen, Chen, Davis, Dean, Devin,
  Ghemawat, Irving, Isard, Kudlur, Levenberg, Monga, Moore, Murray, Steiner,
  Tucker, Vasudevan, Warden, Wicke, Yu, and Zheng]{199317}
Mart{\'\i}n Abadi, Paul Barham, Jianmin Chen, Zhifeng Chen, Andy Davis, Jeffrey
  Dean, Matthieu Devin, Sanjay Ghemawat, Geoffrey Irving, Michael Isard,
  Manjunath Kudlur, Josh Levenberg, Rajat Monga, Sherry Moore, Derek~G. Murray,
  Benoit Steiner, Paul Tucker, Vijay Vasudevan, Pete Warden, Martin Wicke, Yuan
  Yu, and Xiaoqiang Zheng.
\newblock Tensorflow: A system for large-scale machine learning.
\newblock In \emph{12th {USENIX} Symposium on Operating Systems Design and
  Implementation ({OSDI} 16)}, pages 265--283, Savannah, GA, 2016. {USENIX}
  Association.
\newblock ISBN 978-1-931971-33-1.

\bibitem[Attouch et~al.(2010)Attouch, Bolte, Redont, and
  Soubeyran]{attouch2010proximal}
H{\'e}dy Attouch, J{\'e}r{\^o}me Bolte, Patrick Redont, and Antoine Soubeyran.
\newblock Proximal alternating minimization and projection methods for
  nonconvex problems: An approach based on the kurdyka-{\l}ojasiewicz
  inequality.
\newblock \emph{Mathematics of Operations Research}, 35\penalty0 (2):\penalty0
  438--457, 2010.

\bibitem[Blanchard et~al.(2017)Blanchard, Guerraoui, Stainer,
  et~al.]{blanchard2017}
Peva Blanchard, Rachid Guerraoui, Julien Stainer, et~al.
\newblock Machine learning with adversaries: Byzantine tolerant gradient
  descent.
\newblock In \emph{Advances in Neural Information Processing Systems}, pages
  118--128, 2017.

\bibitem[Bottou et~al.(2016)Bottou, Curtis, and Nocedal]{bottou2016}
L{\'e}on Bottou, Frank~E Curtis, and Jorge Nocedal.
\newblock Optimization methods for large-scale machine learning.
\newblock \emph{arXiv preprint arXiv:1606.04838}, 2016.

\bibitem[Chen et~al.(2017)Chen, Su, and Xu]{chen2017}
Yudong Chen, Lili Su, and Jiaming Xu.
\newblock Distributed statistical machine learning in adversarial settings:
  Byzantine gradient descent.
\newblock \emph{arXiv preprint arXiv:1705.05491}, 2017.

\bibitem[Cipar et~al.(2013)Cipar, Ho, Kim, Lee, Ganger, Gibson, Keeton, and
  Xing]{cipar-stragglers}
James Cipar, Qirong Ho, Jin~Kyu Kim, Seunghak Lee, Gregory~R. Ganger, Garth
  Gibson, Kimberly Keeton, and Eric Xing.
\newblock Solving the straggler problem with bounded staleness.
\newblock In \emph{Proceedings of the 14th USENIX Conference on Hot Topics in
  Operating Systems}, HotOS'13, pages 22--22, Berkeley, CA, USA, 2013. USENIX
  Association.

\bibitem[Courbariaux et~al.(2014)Courbariaux, Bengio, and
  David]{DBLP:journals/corr/CourbariauxBD14}
Matthieu Courbariaux, Yoshua Bengio, and Jean{-}Pierre David.
\newblock Low precision arithmetic for deep learning.
\newblock \emph{CoRR}, abs/1412.7024, 2014.

\bibitem[Cui et~al.(2014)Cui, Cipar, Ho, Kim, Lee, Kumar, Wei, Dai, Ganger,
  Gibbons, Gibson, and Xing]{cui-staleness}
Henggang Cui, James Cipar, Qirong Ho, Jin~Kyu Kim, Seunghak Lee, Abhimanu
  Kumar, Jinliang Wei, Wei Dai, Gregory~R. Ganger, Phillip~B. Gibbons, Garth~A.
  Gibson, and Eric~P. Xing.
\newblock Exploiting bounded staleness to speed up big data analytics.
\newblock In \emph{Proceedings of the 2014 USENIX Conference on USENIX Annual
  Technical Conference}, USENIX ATC'14, pages 37--48, Berkeley, CA, USA, 2014.
  USENIX Association.
\newblock ISBN 978-1-931971-10-2.

\bibitem[Dai et~al.(2015)Dai, Kumar, Wei, Ho, Gibson, and
  Xing]{Dai:2015:HDM:2887007.2887019}
Wei Dai, Abhimanu Kumar, Jinliang Wei, Qirong Ho, Garth Gibson, and Eric~P.
  Xing.
\newblock High-performance distributed ml at scale through parameter server
  consistency models.
\newblock In \emph{Proceedings of the Twenty-Ninth AAAI Conference on
  Artificial Intelligence}, AAAI'15, pages 79--87. AAAI Press, 2015.
\newblock ISBN 0-262-51129-0.

\bibitem[Daly(2006)]{daly-optimum-checkpoint}
J.~T. Daly.
\newblock A higher order estimate of the optimum checkpoint interval for
  restart dumps.
\newblock \emph{Future Gener. Comput. Syst.}, 22\penalty0 (3):\penalty0
  303--312, February 2006.
\newblock ISSN 0167-739X.
\newblock \doi{10.1016/j.future.2004.11.016}.

\bibitem[Damaskinos et~al.(2018)Damaskinos, Mhamdi, Guerraoui, Patra, and
  Taziki]{damaskinos2018}
Georgios Damaskinos, El~Mahdi~El Mhamdi, Rachid Guerraoui, Rhicheek Patra, and
  Mahsa Taziki.
\newblock Asynchronous {B}yzantine machine learning (the case of {SGD}).
\newblock In Jennifer Dy and Andreas Krause, editors, \emph{Proceedings of the
  35th International Conference on Machine Learning}, volume~80 of
  \emph{Proceedings of Machine Learning Research}, pages 1145--1154,
  Stockholmsmässan, Stockholm Sweden, 10--15 Jul 2018. PMLR.

\bibitem[Dean et~al.(2012)Dean, Corrado, Monga, Chen, Devin, Le, Mao, Ranzato,
  Senior, Tucker, Yang, and Ng]{Dean:2012:LSD:2999134.2999271}
Jeffrey Dean, Greg~S. Corrado, Rajat Monga, Kai Chen, Matthieu Devin, Quoc~V.
  Le, Mark~Z. Mao, Marc'Aurelio Ranzato, Andrew Senior, Paul Tucker, Ke~Yang,
  and Andrew~Y. Ng.
\newblock Large scale distributed deep networks.
\newblock In \emph{Proceedings of the 25th International Conference on Neural
  Information Processing Systems - Volume 1}, NIPS'12, pages 1223--1231, USA,
  2012. Curran Associates Inc.

\bibitem[Dheeru and Karra~Taniskidou(2017)]{Dua:2017}
Dua Dheeru and Efi Karra~Taniskidou.
\newblock {UCI} machine learning repository, 2017.

\bibitem[Du et~al.(2017)Du, Jin, Lee, Jordan, Poczos, and Singh]{du2017}
Simon~S Du, Chi Jin, Jason~D Lee, Michael~I Jordan, Barnabas Poczos, and Aarti
  Singh.
\newblock Gradient descent can take exponential time to escape saddle points.
\newblock \emph{arXiv preprint arXiv:1705.10412}, 2017.

\bibitem[El~Gamal and Lai(2017)]{elgamal2017}
Mostafa El~Gamal and Lifeng Lai.
\newblock On randomized distributed coordinate descent with quantized updates.
\newblock In \emph{Information Sciences and Systems (CISS), 2017 51st Annual
  Conference on}, pages 1--5. IEEE, 2017.

\bibitem[Gabrilovich et~al.(2013)Gabrilovich, Ringgaard, and
  Subramanya]{clueweb12}
Evgeniy Gabrilovich, Michael Ringgaard, and Amarnag Subramanya.
\newblock Facc1: Freebase annotation of clueweb corpora, version 1 (release
  date 2013-06-26, format version 1, correction level 0).
\newblock http://lemurproject.org/clueweb12/, 2013.

\bibitem[Ge et~al.(2015)Ge, Huang, Jin, and Yuan]{ge2015saddle}
Rong Ge, Furong Huang, Chi Jin, and Yang Yuan.
\newblock Escaping from saddle points --- online stochastic gradient for tensor
  decomposition.
\newblock In Peter Gr{\"u}nwald, Elad Hazan, and Satyen Kale, editors,
  \emph{Proceedings of The 28th Conference on Learning Theory}, volume~40 of
  \emph{Proceedings of Machine Learning Research}, pages 797--842, Paris,
  France, 03--06 Jul 2015. PMLR.

\bibitem[Goldberg et~al.(2001)Goldberg, Roeder, Gupta, and
  Perkins]{Goldberg:2001:ECT:593963.594023}
Ken Goldberg, Theresa Roeder, Dhruv Gupta, and Chris Perkins.
\newblock Eigentaste: A constant time collaborative filtering algorithm.
\newblock \emph{Inf. Retr.}, 4\penalty0 (2):\penalty0 133--151, July 2001.
\newblock ISSN 1386-4564.
\newblock \doi{10.1023/A:1011419012209}.

\bibitem[Guerraoui et~al.(2018)Guerraoui, Rouault, et~al.]{guerraoui2018hidden}
Rachid Guerraoui, S{\'e}bastien Rouault, et~al.
\newblock The hidden vulnerability of distributed learning in byzantium.
\newblock In \emph{International Conference on Machine Learning}, pages
  3518--3527, 2018.

\bibitem[Gupta et~al.(2015)Gupta, Agrawal, Gopalakrishnan, and
  Narayanan]{DBLP:journals/corr/GuptaAGN15}
Suyog Gupta, Ankur Agrawal, Kailash Gopalakrishnan, and Pritish Narayanan.
\newblock Deep learning with limited numerical precision.
\newblock \emph{CoRR}, abs/1502.02551, 2015.

\bibitem[Harlap et~al.(2016)Harlap, Cui, Dai, Wei, Ganger, Gibbons, Gibson, and
  Xing]{harlap-stragglers}
Aaron Harlap, Henggang Cui, Wei Dai, Jinliang Wei, Gregory~R. Ganger,
  Phillip~B. Gibbons, Garth~A. Gibson, and Eric~P. Xing.
\newblock Addressing the straggler problem for iterative convergent parallel
  ml.
\newblock In \emph{Proceedings of the Seventh ACM Symposium on Cloud
  Computing}, SoCC '16, pages 98--111, New York, NY, USA, 2016. ACM.
\newblock ISBN 978-1-4503-4525-5.
\newblock \doi{10.1145/2987550.2987554}.

\bibitem[Harlap et~al.(2017)Harlap, Tumanov, Chung, Ganger, and
  Gibbons]{Harlap:2017:PAM:3064176.3064182}
Aaron Harlap, Alexey Tumanov, Andrew Chung, Gregory~R. Ganger, and Phillip~B.
  Gibbons.
\newblock Proteus: Agile ml elasticity through tiered reliability in dynamic
  resource markets.
\newblock In \emph{Proceedings of the Twelfth European Conference on Computer
  Systems}, EuroSys '17, pages 589--604, New York, NY, USA, 2017. ACM.
\newblock ISBN 978-1-4503-4938-3.
\newblock \doi{10.1145/3064176.3064182}.

\bibitem[Harper and Konstan(2015)]{Harper:2015:MDH:2866565.2827872}
F.~Maxwell Harper and Joseph~A. Konstan.
\newblock The movielens datasets: History and context.
\newblock \emph{ACM Trans. Interact. Intell. Syst.}, 5\penalty0 (4):\penalty0
  19:1--19:19, December 2015.
\newblock ISSN 2160-6455.
\newblock \doi{10.1145/2827872}.

\bibitem[Higham(2002)]{Higham:2002:ASN:579525}
Nicholas~J. Higham.
\newblock \emph{Accuracy and Stability of Numerical Algorithms}.
\newblock Society for Industrial and Applied Mathematics, Philadelphia, PA,
  USA, 2nd edition, 2002.
\newblock ISBN 0898715210.

\bibitem[Hindman et~al.(2011)Hindman, Konwinski, Zaharia, Ghodsi, Joseph, Katz,
  Shenker, and Stoica]{mesos}
Benjamin Hindman, Andy Konwinski, Matei Zaharia, Ali Ghodsi, Anthony~D. Joseph,
  Randy Katz, Scott Shenker, and Ion Stoica.
\newblock Mesos: A platform for fine-grained resource sharing in the data
  center.
\newblock In \emph{Proceedings of the 8th USENIX Conference on Networked
  Systems Design and Implementation}, NSDI'11, pages 295--308, Berkeley, CA,
  USA, 2011. USENIX Association.

\bibitem[Ho et~al.(2013)Ho, Cipar, Cui, Kim, Lee, Gibbons, Gibson, Ganger, and
  Xing]{ho-ssp}
Qirong Ho, James Cipar, Henggang Cui, Jin~Kyu Kim, Seunghak Lee, Phillip~B.
  Gibbons, Garth~A. Gibson, Gregory~R. Ganger, and Eric~P. Xing.
\newblock More effective distributed ml via a stale synchronous parallel
  parameter server.
\newblock In \emph{Proceedings of the 26th International Conference on Neural
  Information Processing Systems - Volume 1}, NIPS'13, pages 1223--1231, USA,
  2013. Curran Associates Inc.

\bibitem[Hubara et~al.(2017)Hubara, Courbariaux, Soudry, El-Yaniv, and
  Bengio]{hubara-qnn}
Itay Hubara, Matthieu Courbariaux, Daniel Soudry, Ran El-Yaniv, and Yoshua
  Bengio.
\newblock Quantized neural networks: Training neural networks with low
  precision weights and activations.
\newblock \emph{J. Mach. Learn. Res.}, 18\penalty0 (1):\penalty0 6869--6898,
  January 2017.
\newblock ISSN 1532-4435.

\bibitem[Hunt et~al.(2010)Hunt, Konar, Junqueira, and
  Reed]{Hunt:2010:ZWC:1855840.1855851}
Patrick Hunt, Mahadev Konar, Flavio~P. Junqueira, and Benjamin Reed.
\newblock Zookeeper: Wait-free coordination for internet-scale systems.
\newblock In \emph{Proceedings of the 2010 USENIX Conference on USENIX Annual
  Technical Conference}, USENIXATC'10, pages 11--11, Berkeley, CA, USA, 2010.
  USENIX Association.

\bibitem[{Jia} et~al.(2018){Jia}, {Song}, {He}, {Wang}, {Rong}, {Zhou}, {Xie},
  {Guo}, {Yang}, {Yu}, {Chen}, {Hu}, {Shi}, and {Chu}]{2018arXiv180711205J}
X.~{Jia}, S.~{Song}, W.~{He}, Y.~{Wang}, H.~{Rong}, F.~{Zhou}, L.~{Xie},
  Z.~{Guo}, Y.~{Yang}, L.~{Yu}, T.~{Chen}, G.~{Hu}, S.~{Shi}, and X.~{Chu}.
\newblock {Highly Scalable Deep Learning Training System with Mixed-Precision:
  Training ImageNet in Four Minutes}.
\newblock \emph{ArXiv e-prints}, July 2018.

\bibitem[Jin et~al.(2017)Jin, Ge, Netrapalli, Kakade, and Jordan]{jin2017}
Chi Jin, Rong Ge, Praneeth Netrapalli, Sham~M Kakade, and Michael~I Jordan.
\newblock How to escape saddle points efficiently.
\newblock \emph{arXiv preprint arXiv:1703.00887}, 2017.

\bibitem[Kingma and Ba(2014)]{DBLP:journals/corr/KingmaB14}
Diederik~P. Kingma and Jimmy Ba.
\newblock Adam: {A} method for stochastic optimization.
\newblock \emph{CoRR}, abs/1412.6980, 2014.

\bibitem[Lakshman and Malik(2010)]{Lakshman:2010:CDS:1773912.1773922}
Avinash Lakshman and Prashant Malik.
\newblock Cassandra: A decentralized structured storage system.
\newblock \emph{SIGOPS Oper. Syst. Rev.}, 44\penalty0 (2):\penalty0 35--40,
  April 2010.
\newblock ISSN 0163-5980.
\newblock \doi{10.1145/1773912.1773922}.

\bibitem[Lang(1995)]{Lang95}
Ken Lang.
\newblock Newsweeder: Learning to filter netnews.
\newblock In \emph{Proceedings of the Twelfth International Conference on
  Machine Learning}, pages 331--339, 1995.

\bibitem[Lecun et~al.(1998)Lecun, Bottou, Bengio, and Haffner]{726791}
Y.~Lecun, L.~Bottou, Y.~Bengio, and P.~Haffner.
\newblock Gradient-based learning applied to document recognition.
\newblock \emph{Proceedings of the IEEE}, 86\penalty0 (11):\penalty0
  2278--2324, Nov 1998.
\newblock ISSN 0018-9219.
\newblock \doi{10.1109/5.726791}.

\bibitem[Lewis et~al.(2004)Lewis, Yang, Rose, and
  Li]{Lewis:2004:RNB:1005332.1005345}
David~D. Lewis, Yiming Yang, Tony~G. Rose, and Fan Li.
\newblock Rcv1: A new benchmark collection for text categorization research.
\newblock \emph{J. Mach. Learn. Res.}, 5:\penalty0 361--397, December 2004.
\newblock ISSN 1532-4435.

\bibitem[Li et~al.(2014{\natexlab{a}})Li, Andersen, Park, Smola, Ahmed,
  Josifovski, Long, Shekita, and Su]{186214}
Mu~Li, David~G. Andersen, Jun~Woo Park, Alexander~J. Smola, Amr Ahmed, Vanja
  Josifovski, James Long, Eugene~J. Shekita, and Bor-Yiing Su.
\newblock Scaling distributed machine learning with the parameter server.
\newblock In \emph{11th {USENIX} Symposium on Operating Systems Design and
  Implementation ({OSDI} 14)}, pages 583--598, Broomfield, CO,
  2014{\natexlab{a}}. {USENIX} Association.
\newblock ISBN 978-1-931971-16-4.

\bibitem[Li et~al.(2014{\natexlab{b}})Li, Andersen, Smola, and Yu]{li-ps-nips}
Mu~Li, David~G. Andersen, Alexander Smola, and Kai Yu.
\newblock Communication efficient distributed machine learning with the
  parameter server.
\newblock In \emph{Proceedings of the 27th International Conference on Neural
  Information Processing Systems - Volume 1}, NIPS'14, pages 19--27, Cambridge,
  MA, USA, 2014{\natexlab{b}}. MIT Press.

\bibitem[Liu(1994)]{10.2307/2290921}
Jun~S. Liu.
\newblock The collapsed gibbs sampler in bayesian computations with
  applications to a gene regulation problem.
\newblock \emph{Journal of the American Statistical Association}, 89\penalty0
  (427):\penalty0 958--966, 1994.
\newblock ISSN 01621459.

\bibitem[Low et~al.(2012)Low, Bickson, Gonzalez, Guestrin, Kyrola, and
  Hellerstein]{Low:2012:DGF:2212351.2212354}
Yucheng Low, Danny Bickson, Joseph Gonzalez, Carlos Guestrin, Aapo Kyrola, and
  Joseph~M. Hellerstein.
\newblock Distributed graphlab: A framework for machine learning and data
  mining in the cloud.
\newblock \emph{Proc. VLDB Endow.}, 5\penalty0 (8):\penalty0 716--727, April
  2012.
\newblock ISSN 2150-8097.
\newblock \doi{10.14778/2212351.2212354}.

\bibitem[Mania et~al.(2015)Mania, Pan, Papailiopoulos, Recht, Ramchandran, and
  Jordan]{mania2015}
Horia Mania, Xinghao Pan, Dimitris Papailiopoulos, Benjamin Recht, Kannan
  Ramchandran, and Michael~I Jordan.
\newblock Perturbed iterate analysis for asynchronous stochastic optimization.
\newblock \emph{arXiv preprint arXiv:1507.06970}, 2015.

\bibitem[Nair and Hinton(2010)]{Nair:2010:RLU:3104322.3104425}
Vinod Nair and Geoffrey~E. Hinton.
\newblock Rectified linear units improve restricted boltzmann machines.
\newblock In \emph{Proceedings of the 27th International Conference on
  International Conference on Machine Learning}, ICML'10, pages 807--814, USA,
  2010. Omnipress.
\newblock ISBN 978-1-60558-907-7.

\bibitem[Nemirovski et~al.(2009)Nemirovski, Juditsky, Lan, and
  Shapiro]{nemirovski2009robust}
Arkadi Nemirovski, Anatoli Juditsky, Guanghui Lan, and Alexander Shapiro.
\newblock Robust stochastic approximation approach to stochastic programming.
\newblock \emph{SIAM Journal on optimization}, 19\penalty0 (4):\penalty0
  1574--1609, 2009.

\bibitem[Nesterov(2013)]{nesterov2013book}
Yurii Nesterov.
\newblock \emph{Introductory lectures on convex optimization: A basic course},
  volume~87.
\newblock Springer Science \& Business Media, 2013.

\bibitem[Niu et~al.(2011)Niu, Recht, Re, and
  Wright]{Niu:2011:HLA:2986459.2986537}
Feng Niu, Benjamin Recht, Christopher Re, and Stephen~J. Wright.
\newblock Hogwild!: A lock-free approach to parallelizing stochastic gradient
  descent.
\newblock In \emph{Proceedings of the 24th International Conference on Neural
  Information Processing Systems}, NIPS'11, pages 693--701, USA, 2011. Curran
  Associates Inc.
\newblock ISBN 978-1-61839-599-3.

\bibitem[Qiao et~al.(2018)Qiao, Aghayev, Yu, Chen, Ho, Gibson, and
  Xing]{216041}
Aurick Qiao, Abutalib Aghayev, Weiren Yu, Haoyang Chen, Qirong Ho, Garth~A.
  Gibson, and Eric~P. Xing.
\newblock Litz: Elastic framework for high-performance distributed machine
  learning.
\newblock In \emph{2018 {USENIX} Annual Technical Conference ({USENIX} {ATC}
  18)}, pages 631--644, Boston, MA, 2018. {USENIX} Association.
\newblock ISBN 978-1-931971-44-7.

\bibitem[Rakhlin et~al.(2012)Rakhlin, Shamir, Sridharan,
  et~al.]{rakhlin2012making}
Alexander Rakhlin, Ohad Shamir, Karthik Sridharan, et~al.
\newblock Making gradient descent optimal for strongly convex stochastic
  optimization.
\newblock In \emph{ICML}, volume~12, pages 1571--1578. Citeseer, 2012.

\bibitem[Sandberg et~al.(1988)Sandberg, Golgberg, Kleiman, Walsh, and
  Lyon]{Sandberg:1988:DIS:59309.59338}
R.~Sandberg, D.~Golgberg, S.~Kleiman, D.~Walsh, and B.~Lyon.
\newblock Innovations in internetworking.
\newblock chapter Design and Implementation of the Sun Network Filesystem,
  pages 379--390. Artech House, Inc., Norwood, MA, USA, 1988.
\newblock ISBN 0-89006-337-0.

\bibitem[Sergeev and Balso(2018)]{DBLP:journals/corr/abs-1802-05799}
Alexander Sergeev and Mike~Del Balso.
\newblock Horovod: fast and easy distributed deep learning in tensorflow.
\newblock \emph{CoRR}, abs/1802.05799, 2018.

\bibitem[Vavilapalli et~al.(2013)Vavilapalli, Murthy, Douglas, Agarwal, Konar,
  Evans, Graves, Lowe, Shah, Seth, Saha, Curino, O'Malley, Radia, Reed, and
  Baldeschwieler]{yarn}
Vinod~Kumar Vavilapalli, Arun~C. Murthy, Chris Douglas, Sharad Agarwal, Mahadev
  Konar, Robert Evans, Thomas Graves, Jason Lowe, Hitesh Shah, Siddharth Seth,
  Bikas Saha, Carlo Curino, Owen O'Malley, Sanjay Radia, Benjamin Reed, and
  Eric Baldeschwieler.
\newblock Apache hadoop yarn: Yet another resource negotiator.
\newblock In \emph{Proceedings of the 4th Annual Symposium on Cloud Computing},
  SOCC '13, pages 5:1--5:16, New York, NY, USA, 2013. ACM.
\newblock ISBN 978-1-4503-2428-1.
\newblock \doi{10.1145/2523616.2523633}.

\bibitem[Wei et~al.(2015)Wei, Dai, Qiao, Ho, Cui, Ganger, Gibbons, Gibson, and
  Xing]{Wei:2015:MCC:2806777.2806778}
Jinliang Wei, Wei Dai, Aurick Qiao, Qirong Ho, Henggang Cui, Gregory~R. Ganger,
  Phillip~B. Gibbons, Garth~A. Gibson, and Eric~P. Xing.
\newblock Managed communication and consistency for fast data-parallel
  iterative analytics.
\newblock In \emph{Proceedings of the Sixth ACM Symposium on Cloud Computing},
  SoCC '15, pages 381--394, New York, NY, USA, 2015. ACM.
\newblock ISBN 978-1-4503-3651-2.
\newblock \doi{10.1145/2806777.2806778}.

\bibitem[Weil et~al.(2006)Weil, Brandt, Miller, Long, and
  Maltzahn]{Weil:2006:CSH:1298455.1298485}
Sage~A. Weil, Scott~A. Brandt, Ethan~L. Miller, Darrell D.~E. Long, and Carlos
  Maltzahn.
\newblock Ceph: A scalable, high-performance distributed file system.
\newblock In \emph{Proceedings of the 7th Symposium on Operating Systems Design
  and Implementation}, OSDI '06, pages 307--320, Berkeley, CA, USA, 2006.
  USENIX Association.
\newblock ISBN 1-931971-47-1.

\bibitem[Xu and Yin(2017)]{xu2017globally}
Yangyang Xu and Wotao Yin.
\newblock A globally convergent algorithm for nonconvex optimization based on
  block coordinate update.
\newblock \emph{Journal of Scientific Computing}, 72\penalty0 (2):\penalty0
  700--734, 2017.

\bibitem[Zhang et~al.(2017{\natexlab{a}})Zhang, Li, Kara, Alistarh, Liu, and
  Zhang]{zipml}
Hantian Zhang, Jerry Li, Kaan Kara, Dan Alistarh, Ji~Liu, and Ce~Zhang.
\newblock {Z}ip{ML}: Training linear models with end-to-end low precision, and
  a little bit of deep learning.
\newblock In Doina Precup and Yee~Whye Teh, editors, \emph{Proceedings of the
  34th International Conference on Machine Learning}, volume~70 of
  \emph{Proceedings of Machine Learning Research}, pages 4035--4043,
  International Convention Centre, Sydney, Australia, 06--11 Aug
  2017{\natexlab{a}}. PMLR.

\bibitem[Zhang et~al.(2017{\natexlab{b}})Zhang, Zheng, Xu, Dai, Ho, Liang, Hu,
  Wei, Xie, and Xing]{poseidon}
Hao Zhang, Zeyu Zheng, Shizhen Xu, Wei Dai, Qirong Ho, Xiaodan Liang, Zhiting
  Hu, Jinliang Wei, Pengtao Xie, and Eric~P. Xing.
\newblock Poseidon: An efficient communication architecture for distributed
  deep learning on gpu clusters.
\newblock In \emph{Proceedings of the 2017 USENIX Conference on Usenix Annual
  Technical Conference}, USENIX ATC '17, pages 181--193, Berkeley, CA, USA,
  2017{\natexlab{b}}. USENIX Association.
\newblock ISBN 978-1-931971-38-6.

\end{thebibliography}

\appendix

% !TEX root = ./main.tex

\section{Proofs}

\subsection{Proof of Lemma~\ref{lem:gd:mainassumption}}
\label{appendix:lem:gd:mainassumption}

Lemma~\ref{lem:gd:mainassumption} follows from standard results on gradient descent, see e.g. the proof of Theorem~2.1.5 in \citet{nesterov2013book}.

\subsection{Proof of Theorem \ref{thm:main}}
\label{appendix:thm:main}

We start with the following useful lemma:
\begin{lemma}
\label{lem:periodic:basic}
Assuming \eqref{assm:fseq:stepbound}, we have for any $k$
\begin{align}
\label{eq:periodic:basic}
\E\norm{\fseq^{(k+1)} - \opt}
% &\le \underbrace{c^{k+1}\norm{x^{(0)} - \opt}}_{\text{(A)}} + \underbrace{\sum_{\ell=0}^{k+1}c^{k-\ell+1}\norm{\delta_{\ell}}}_{\text{(B)}} \\
&\le c^{k+1}\Big[\norm{x^{(0)} - \opt} + \sum_{\ell=0}^{k}c^{-\ell}\E\norm{\delta_{\ell}}\Big].
\end{align}
\end{lemma}

\begin{proof}
For any $k>0$ we have
\begin{align}
\E\norm{\fseq^{(k+1)} - \opt}
\nonumber
&= \E\norm{f(\corrupt^{(k)}) - \opt} \\
% &=\norm{x^{(k)} + \delta_{k} - \opt} \\
% &\le\norm{x^{(k)} - \opt} + \norm{\delta_{k}} \\
\nonumber
&\le c\E\norm{\corrupt^{(k)} - \opt} \\
\nonumber
&= c\E\norm{\fseq^{(k)} + \delta_{k} - \opt} \\
\label{eq:lem:periodic:basic:1d}
&\le c\big[\E\norm{\fseq^{(k)} - \opt} + \E\norm{\delta_{k}}\big],
\end{align}

\noindent
where we have invoked \eqref{assm:fseq:stepbound}. Iterating this inequality, we obtain:
\begin{align}
c\big[\E\norm{\fseq^{(k)} &- \opt} + \E\norm{\delta_{k}}\big] \\
% &\le c\big[c\big[\E\norm{\fseq^{(k-1)} - \opt} + \E\norm{\delta_{k-1}}\big] + \E\norm{\delta_{k}}\big] \\
\nonumber
&\le c^{2}\E\norm{\fseq^{(k-1)} - \opt} + c^{2}\E\norm{\delta_{k-1}} + c\E\norm{\delta_{k}} \\
\nonumber
&\,\,\,\,\,\,\vdots\\
\nonumber
&\le c^{k+1}\E\norm{\fseq^{(0)} - \opt} + \sum_{i=0}^{k}c^{i+1}\E\norm{\delta_{k-i}} \\
\label{eq:lem:periodic:basic:2d}
&= c^{k+1}\norm{\seq^{(0)} - \opt} + \sum_{\ell=0}^{k}c^{k-\ell+1}\E\norm{\delta_{\ell}}.
\end{align}

\noindent
In the last step we simply re-indexed the summation and use $\fseq^{(0)}=\seq^{(0)}$. Combining \eqref{eq:lem:periodic:basic:1d} and \eqref{eq:lem:periodic:basic:2d} yields the desired bound.
\end{proof}

\begin{proof}[Proof of Theorem~\ref{thm:main}]
By Lemma~\ref{lem:periodic:basic}, we have for any $k>T$,
\begin{align}
\label{eq:single:deriv:1}
\E\norm{y^{(k)} - \opt}
\le c^{k}\Big[\norm{x^{(0)} - \opt} + \sum_{\ell=0}^{T}c^{-\ell}\E\norm{\delta_{\ell}}\Big]
&<\eps \\
\label{eq:single:deriv:2}
\iff
\frac1\eps\Big[\norm{x^{(0)} - \opt} + \Delta_{T}\Big]
&< c^{-k}
% \label{eq:single:deriv:3}
% &\iff
% \log\Big(\frac1\eps\Big[\norm{x^{(0)} - \opt} + \Delta_{T}\Big]\Big)
% < -k\log c
% = k\log(1/c).
\end{align}

\noindent
Re-arranging, we deduce that $\E\norm{y^{(k)} - \opt}<\eps$ if 
\begin{align*}
k
> \frac{\log\Big(\frac1\eps\Big[\norm{x^{(0)} - \opt} + \Delta_{T}\Big]\Big)}{\log(1/c)}
\ge \iter(\fseq^{(k)},\;\eps).
\end{align*}

It is easy to check (e.g. take $\delta_{k}=0$ in the previous derivation) that 
$\iter(\seq^{(k)},\;\eps) = \log\big(\frac1\eps\norm{x^{(0)} - \opt}\big)/\log(1/c)$ is a bound on the number of iterations required for the unperturbed sequence $\seq^{(k)}$ to reach $\eps$-optimality. Thus, the iteration cost is given by
\begin{align*}
\cost(\delta_{k},\;\eps)
&= \iter(\fseq^{(k)},\;\eps) - \iter(\seq^{(k)},\;\eps) \\
% = \frac{\log\Big(\frac1\eps\Big[\norm{x^{(0)} - \opt} + c^{-T}\norm{\delta_{T}}\Big]\Big)}{\log(1/c)} - \frac{\log\big(\frac1\eps\norm{x^{(0)} - \opt}\big)}{\log(1/c)} \\
&\le \frac{\log\Big(\frac1\eps\Big[\norm{x^{(0)} - \opt} + \Delta_{T}\Big]\Big) - \log\big(\frac1\eps\norm{x^{(0)} - \opt}\big)}{\log(1/c)} \\
% &= \frac{\log\Big(\frac{\norm{x^{(0)} - \opt} + c^{-T}\norm{\delta_{T}}}{\norm{x^{(0)} - \opt}}\Big)}{\log(1/c)}.
&= \frac{\log\Big(1 + \frac{\Delta_{T}}{\norm{x^{(0)} - \opt}}\Big)}{\log(1/c)},
\end{align*}

\noindent
as claimed.
\end{proof}

\subsection{Proof of Theorem \ref{thm:partial-recovery}}
\label{appendix:thm:partial-recovery}

Let \( z = x^{(C)} \) be the checkpoint of the model parameters saved at iteration \( C \), and let \( S \) be the subset of model parameters lost during a failure at iteration \( T \). Then
\[ ||\delta|| = ||z - x^{(T)}|| \]
is the perturbation due to full recovery, and
\[ ||\delta'|| = ||z_S - x^{(T)}_S|| \]
is the perturbation due to partial recovery, since \( x^{(T)}_{S^c} \) does not change due to failure, where \( S^c \) is the complement set of \( S \). Then we have
\begin{align*}
||\delta'||^2 &= ||z_S - x^{(T)}_S||^2 \\
&\le ||z_S - x^{(T)}_S||^2 + ||z_{S^c} - x^{(T)}_{S^c}||^2 \\
&\quad + (z_S - x^{(T)}_S) \cdot (z_{S^c} - x^{(T)}_{S^c}) \\
&\le ||(z_S - x^{(T)}_S) + (z_{S^c} - x^{(T)}_{S^c})||^2 \\
&= ||z - x^{(T)}||^2 = ||\delta||^2
\end{align*}
Thus \( \norm{\delta'} \le \norm{\delta} \), as claimed.

\subsection{Proof of Theorem \ref{thm:partial-recovery-exp}}
\label{appendix:thm:partial-recovery-exp}

Let \( z = x^{(C)} \) be the checkpoint of the model parameters saved at iteration \( C \), and let \( S \) be the subset (chosen uniformly at random) of model parameters lost during a failure at iteration \( T \). Then
\begin{align*}
\mathbb{E}||\delta'||^2 &= \mathbb{E}||z_S - x^{(T)}_S||^2 \\
&= \mathbb{E}\left[(z_S - x^{(T)}_S) \cdot (z_S - x^{(T)}_S)\right] \\
&= \sum_i \mathbb{E}\left[(z_S - x^{(T)}_S)_i^2\right] \\
&= \sum_i \mathbb{E}\left[[i \in S](z_i - x^{(T)}_i)^2\right] \\
&= \sum_i P(i \in S)(z_i - x^{(T)}_i)^2 \\
&= \sum_i p(z_i - x^{(T)}_i)^2 \\
&= p||z - x^{(T)}||^2 = p||\delta||^2
\end{align*}
Thus \( \mathbb{E}||\delta'||^2 = p||\delta||^2 \), as claimed.

\section{Extensions}

This appendix discusses extensions of our framework to (a) Infinite perturbations (Example~\ref{ex:analysis:infinite}) and (b) SGD (Example~\ref{ex:analysis:sgd}).

\subsection{Analysis for $T=\infty$}
\label{appendix:infinite}

Suppose $\E\norm{\delta_{k}}\le\Delta$ for each $k$. For intuition, note that Lemma~\ref{lem:periodic:basic} implies that for any $k$,
\begin{align}
\E\norm{\fseq^{(k+1)} - \opt}
\nonumber
&\le c^{k+1}\Big[\norm{x^{(0)} - \opt} + \sum_{\ell=0}^{k}c^{-\ell}\Delta\Big] \\
\label{eq:infinite:basic:1b}
&= c^{k+1}\Big[\norm{x^{(0)} - \opt} + \Delta\frac{1-c^{-(k+1)}}{1-c^{-1}}\Big] \\
% &= c^{k+1}\norm{x^{(0)} - \opt} + \Delta\frac{c^{k+1}-1}{1-c^{-1}}
\nonumber
&= c^{k+1}\norm{x^{(0)} - \opt} + \Delta\frac{c-c^{k+2}}{1-c} \\
\nonumber
&\overset{k\to\infty}{\longrightarrow} \frac{c}{1-c}\Delta.
\end{align}
\noindent
Evidently, there is an irreducible, positive error if we are subject to fault in every single iteration.

Thus, the best we can hope for is convergence to within some tolerance $\eps>(c/(1-c))\Delta$.
Re-arranging and solving for $k$ in \eqref{eq:infinite:basic:1b} as in the proof of Theorem~\ref{thm:main}, we deduce that $\E\norm{\fseq^{(k+1)} - \opt}<\eps$ as long as
\begin{align*}
k 
> \frac{\log\Big(\frac{\norm{x^{(0)} - \opt} - \frac{c}{1-c}\Delta}{\eps - \frac{c}{1-c}\Delta}\Big)}{\log(1/c)}.
\end{align*}

\noindent
The resulting iteration cost bound is (cf. \eqref{eq:bound:itercost}):
\begin{align}
\cost(\delta_{k},\;\eps)
&\le \frac{\log\Bigg(\frac{1 - \frac{\frac{c}{1-c}\Delta}{\norm{x^{(0)} - \opt}}}{1 - \frac{\frac{c}{1-c}\Delta}{\eps}}\Bigg)}{\log(1/c)}.
% &= \frac{\log\Bigg(1+\frac{\frac{\frac{1}{1-c}\Delta}{\eps} - \frac{\frac{c}{1-c}\Delta}{\norm{x^{(0)} - \opt}}}{1 - \frac{\frac{1}{1-c}\Delta}{\eps}}\Bigg)}{\log(1/c)} \\
% &= \frac{\log\Bigg(1 + \frac{\frac{\Delta}{1-c}\Big[\frac1\eps - \frac{c}{\norm{x^{(0)} - \opt}}\Big]}{\frac{\Delta}{1-c}\Big[\frac{1-c}{\Delta} - \frac{1}{\eps}\Big]}\Bigg)}{\log(1/c)} \\
% &= \frac{\log\Bigg(1 + \frac{\frac1\eps - \frac{c}{\norm{x^{(0)} - \opt}}}{\frac{1-c}{\Delta} - \frac{1}{\eps}}\Bigg)}{\log(1/c)} \\
\end{align}

\noindent
This bound is only informative if $\norm{x^{(0)} - \opt}>(c/(1-c))\Delta$ and $\eps>(c/(1-c))\Delta$.

\subsection{Stochastic gradient descent}
\label{appendix:sgd}

Assume the objective function $\ell$ is strongly convex, as in Lemma~\ref{lem:gd:mainassumption}. In order to derive upper bounds on the iteration cost for SGD, we start from following general recursion, which is standard from the literature \citep{nemirovski2009robust,rakhlin2012making}:
\begin{align}
\label{eq:sgd:recursion}
\E\norm{\seq^{(k+1)} - \opt}^{2}
\le (1-\alpha_{k})\E\norm{\seq^{(k)} - \opt}^{2} + \alpha_{k}^{2}G^{2},
\end{align}

\noindent
where $\alpha_{k}\to0$ is a sequence that depends on $\ell$ and the step size, and $G$ is an upper bound on the expected norm of the stochastic gradients. Comparing \eqref{eq:sgd:recursion} to \eqref{eq:lem:periodic:basic:1d}, the only difference is that instead of a constant $c<1$, we have a sequence $1-\alpha_{k}\to1$. Thus, instead of decaying at the geometric rate $c^{k}$, the iterates of SGD converge at a slower rate $(1-\alpha_{1})\cdots(1-\alpha_{k})$.

Define $a_{k}:=(1-\alpha_{1})\cdots(1-\alpha_{k})$. Under the assumptions of Theorem~\ref{thm:main}, we have the following analogue of \eqref{eq:single:deriv:1}:
\begin{align*}
% \label{eq:single:deriv:1}
\E\norm{y^{(k)} - \opt}
\le a_{k}\Big[\norm{x^{(0)} - \opt} + \sum_{\ell=0}^{T}a_{\ell}^{-1}\big(\E\norm{\delta_{\ell}} + \alpha_{\ell}^{2}G^{2}\big)\Big]
&<\eps. %\\
% \label{eq:single:deriv:2}
% \iff
% \frac1\eps\Big[\norm{x^{(0)} - \opt} + \Delta_{T}\Big]
% &< c^{-k}
\end{align*}
\noindent
This yields an implicit formula for $k$, which can be used to upper bound the iteration cost for SGD. For example, a popular choice of $\alpha_{k}$ is $\alpha_{k}\propto1/k$, in which case $a_{k}\propto 1/k$ (this follows from an induction argument), and solving for $k$ yields the desired upper bound.

\section{Details of Models and Datasets}
\label{appendix:experiments}

In this brief appendix, we collect some details of the different models and datasets used in the experiments in Section \ref{sec:models}.

\paragraph{Multinomial Logistic Regression (MLR).} For MNIST, we use a batch size of 10,000, a learning rate of \( 1 \times 10^{-5} \), and a convergence criteria of \( 2.5 \times 10^4 \) in cross-entropy loss. For CoverType, we use a batch size of 1,000, a learning rate of \( 1 \times 10^{-7} \), and a convergence criteria of \( 6.7 \times 10^5 \) in cross-entropy loss. For both datasets, the convergence criteria is reached in roughly 60 iterations.

\paragraph{Matrix Factorization (MF).} For MovieLens, we use 20 factors and a convergence criteria of \( 9.2 \times 10^2 \) in mean squared error loss. For Jester, we use 5 factors and a convergence criteria of \( 5.57 \times 10^3 \) in mean squared error loss. The MovieLens dataset is the \texttt{movielens-small} version consisting of 671 users and 9,125 items. The Jester dataset is the \texttt{Jester 2+} version, We further remove users with no ratings, and re-scale ratings from \( [-10, 10] \) to \( [0, 10] \). For both datasets, the factor matrices \( L \) and \( R \) are randomly initialized with each entry sampled uniformly at random from \( [0, 1) \), and the convergence criteria is reached in roughly 60 iterations.

\paragraph{Latent Dirichlet Allocation (LDA).} In LDA, each document in the input data consists of a series of tokens, where each token is assigned a categorical topic. Topic assignments are repeatedly randomly sampled during the lifetime of a job. From these token-topic assignments, a document-topic distribution is constructed for each document, and a word-topic distribution is constructed for each unique word. Both the document-topic and word-topic distributions can be re-generated given the token-topic assignments, so losing the distributions themselves is not a problem. However, when distributed, each document-topic distribution is typically co-located with the document it corresponds to. Thus, losing a document-topic distribution is typically associated with also losing the token-topic assignments of that document, which do require recovery from a saved checkpoint. Therefore, we only consider the loss of document-topic distributions, and assume that the word-topic distributions can be reconstructed at any time.

Since the parameters of LDA are distributions, a natural norm to use is the total variation norm. However, the total variation norm when applied to LDA puts the same weight onto every document-topic distribution. This means that re-sampling a token-topic assignment in a shorter document has a greater impact to the overall norm than re-sampling a token-topic assignment in a longer document, which biases checkpoint prioritization towards shorter documents. To address this, we scale the total variation norm of each document-topic distribution by the length of the document it corresponds to. The result is still a valid norm, since it is a positive linear combination (which is constant with respect to the input data) of total variation norms.

For 20 Newsgroups, we use a convergence criteria of \( 9.5 \times 10^6 \) in negative log-likelihood. For Reuters, we use a convergence criteria of \( 8.5 \times 10^5 \) in negative log-likelihood. For both datasets we train using \( 20 \) topics and hyperparameters \( \alpha = \beta = 1 \). The convergence criteria is reached in roughly \( 60 \) iterations.

\paragraph{Convolutional Neural Network (CNN).} We use a batch size of \( 64 \), the recommended Adam settings of \( \alpha = 0.001 \), \( \beta_1 = 0.9 \), \( \beta_2 = 0.999 \), and \( \epsilon = 10^{-8} \), and a convergence criteria of \( 0.08 \) in cross-entropy loss. In by-layer partitioning, the weight and bias parameters are independent and partitioned separately (so they can either be lost together, or not). In by-shard partitioning, each parameter tensor is evenly partitioning according to its first dimension.

\end{document}